\documentclass[11pt]{article}

\newif\ifanonymous
\anonymousfalse

\usepackage[margin=3cm]{geometry}
\usepackage{amsmath}
\usepackage{amssymb}
\usepackage{amsthm}
\usepackage{times}
\usepackage{amsmath}
\usepackage[skins]{tcolorbox}
\usepackage{amssymb}
\usepackage{pgfplots}
\usepackage{pgfplotstable}
\usepackage{dsfont}
\usepackage{tikz}
\usetikzlibrary{arrows.meta,calc,decorations.markings,math,arrows.meta}
\usepackage{natbib}
\usepackage{setspace}
\usepackage{mathrsfs}
\usepackage{booktabs}
\usepackage{bm}
\usepackage[framemethod=TikZ]{mdframed}
\usepackage{colortbl}
\usepackage{wrapfig}
\usepackage{mathtools}
\usepackage[plain]{algorithm}
\usepackage{float}
\usepackage[textsize=tiny]{todonotes}
\usepackage[font=small,labelfont=bf]{caption}

\definecolor{c1}{HTML}{316AC6}
\definecolor{c2}{HTML}{A88C1A}
\definecolor{c3}{HTML}{4F8436}
\definecolor{c4}{HTML}{8E2F78}
\definecolor{c5}{HTML}{A02026}
\definecolor{c6}{HTML}{1B9996}
\definecolor{c7}{HTML}{68720A}
\definecolor{c8}{HTML}{9E9E9E}
\definecolor{c9}{HTML}{9B0F0F}

\newenvironment{simplealg}{\begin{mdframed}[roundcorner=1pt,backgroundcolor=black!5!white]\setlength{\parindent}{0cm}\setlength{\parskip}{0.2cm}\tt\vspace{-0.15cm}}{\end{mdframed}\vspace{-0.55cm}}
\newcommand{\algind}{\hspace{0.5cm}}
\setcitestyle{square}

\usepackage{hyperref}
\definecolor{dkblue}{cmyk}{1,.54,.04,.19} 
\hypersetup{
      bookmarks=true,         
      unicode=false,          
      pdftoolbar=true,        
      pdfmenubar=true,        
      pdffitwindow=false,     
      pdfstartview={FitH},    
      pdftitle={Partial monitoring},    
      pdfauthor={Lattimore and Szepesvari},     
      pdfsubject={Bandits},   
      pdfcreator={pdflatex},   
      pdfproducer={Producer}, 
      pdfkeywords={bandits} {statistics} {machine learning}, 
      pdfnewwindow=true,      
      colorlinks=true,        
      linkcolor=dkblue,       
      citecolor=dkblue,       
      filecolor=dkblue,       
      urlcolor=dkblue,        
}

\usepackage[capitalise]{cleveref}
\theoremstyle{plain}
\newtheorem{theorem}{Theorem}

\newtheorem{lemma}[theorem]{Lemma}
\newtheorem{proposition}[theorem]{Proposition}

\theoremstyle{definition}

\newtheorem{assumption}[theorem]{Assumption}
\newtheorem{remark}[theorem]{Remark}
\theoremstyle{remark}

\newcommand{\argmin}{\operatornamewithlimits{arg\,min}}

\newcommand{\Prob}[1]{\mathbb{P}\left(#1\right)}
\newcommand{\R}{\mathbb{R}}
\newcommand{\E}{\mathbb{E}}

\newcommand{\cF}{\mathcal F}
\newcommand{\cG}{\mathcal G}

\newcommand{\cD}{\mathcal D}

\newcommand{\cT}{\mathcal T}
\newcommand{\cP}{\mathcal P}

\newcommand{\cL}{\mathcal L}
\newcommand{\cX}{\mathcal X}
\newcommand{\ip}[1]{\langle #1 \rangle}
\newcommand{\one}[1]{\mathds{1}(#1)}
\newcommand{\sind}{\mathds{1}}

\newcommand{\bias}{\operatorname{bias}}

\newcommand{\cH}{\mathcal H}
\newcommand{\cHu}{\mathcal H_{\circ}}

\newcommand{\cE}{\mathcal E}
\newcommand{\ones}{\bm{1}}
\newcommand{\zeros}{\bm{0}}
\newcommand{\set}[1]{\left\{ #1\right\}}
\newcommand{\norm}[1]{\left\Vert #1\right\Vert}

\newcommand{\diam}{\operatorname{diam}}
\newcommand{\diag}{\operatorname{diag}}

\newcommand{\tr}{\operatorname{tr}}
\newcommand{\Reg}{\mathfrak{R}}
\newcommand{\opt}{\operatorname{opt}}
\newcommand{\opthp}{\operatorname{opthp}}
\newcommand{\troot}{\operatorname{root}}
\newcommand{\tpath}{\operatorname{path}}
\newcommand{\est}{w}
\let\epsilon\varepsilon

\title{\Large \textsc{Exploration by Optimisation in Partial Monitoring}}
\author{Tor Lattimore and Csaba Szepesv\'ari \\[0.3cm] DeepMind}
\date{}

\begin{document}

\maketitle

\begin{abstract}
We provide a simple, intuitive and efficient algorithm for adversarial $k$-action $d$-outcome partial monitoring games.
Let $m\le d$ denote the maximum number of different observations per action.
We show that for non-degenerate locally observable games  the $n$-round minimax regret is bounded by 
$2m k^{3/2} \sqrt{3n \log(k)}$, 
matching the best known information-theoretic upper bound in this case.
The same algorithm also achieves near-optimal regret for full information, bandit and
globally observable games. High probability bounds and simple experiments are also provided.
\end{abstract}

\section{Introduction}
Partial monitoring is a generalisation of the bandit framework that decouples the loss and the observations.
The framework is sufficiently rich to model bandits, linear bandits, full information games, dynamic pricing, bandits with graph feedback and many problems between
and beyond these examples.
For positive integer $m$ let $[m] = \{1,\dots,m\}$.
A finite adversarial partial monitoring game is determined by a signal matrix
$\Phi \in \Sigma^{k \times d}$ and loss matrix $\cL \in [0,1]^{k\times d}$ where $\Sigma$ is an arbitrary finite set.
Both $\Phi$ and $\cL$ are known to the learner.
The game proceeds over $n$ rounds. First the adversary chooses a sequence $(x_t)_{t=1}^n$ with $x_t \in [d]$. 
In each round $t\in [n]$
the learner chooses an action $A_t \in [k]$, suffers loss $\cL_{A_t x_t}$, but only observes the signal
$\sigma_t = \Phi_{A_tx_t}$.
The regret is
\begin{align*}
\Reg_n = \max_{a \in [k]} \sum_{t=1}^n \left(\cL_{A_tx_t} - \cL_{ax_t}\right)\,.
\end{align*}
The minimax regret is
\begin{align*}
\Reg_n^* = \inf_\pi \sup_{(x_t)_{t=1}^n} \E\left[\Reg_n\right]\,,
\end{align*}
where the expectation is with respect to the randomness in the actions and
$\pi$ is the policy of the learner mapping action/observation sequences to distributions over the actions.
Our main contribution is a simple and efficient algorithm for finite non-degenerate locally observable partial monitoring games for which
\begin{align}
\Reg_n^* \leq  2k^{3/2} m \sqrt{3n \log(k)}\,. \label{eq:regret}
\end{align}
The same algorithm is adaptive to other types of game, achieving near-optimal regret for globally observable games, a regret of $\sqrt{2nk \log(k)}$ for bandits and $\sqrt{2n \log(k)}$ for full information games.

\begin{wraptable}[8]{r}{6.5cm}
\centering
\vspace{0.2cm}
\begin{tabular}{|ll|}
\hline
Trivial & $\Reg_n^* = 0$ \\
Easy & $\Reg_n^* = \Theta(n^{1/2})$ \\
Hard & $\Reg_n^* = \Theta(n^{2/3})$ \\
Hopeless & $\Reg_n^* = \Omega(n)$ \\ \hline
\end{tabular}
\caption{Classification of finite partial monitoring}
\label{tab:cat}
\end{wraptable}
\paragraph{Related work}
Partial monitoring goes back to the work by \cite{Rus99}, who derived Hannan consistent policies.
There has been significant effort in understanding the dependence of the regret on the horizon. The key result is the classification theorem, showing that all 
finite partial monitoring games lie in one of four categories as illustrated in \cref{tab:cat}.
The classification theorem also gives a procedure to decide into which category a given game belongs. Since the game is known in advance, there is no need to learn
the classification of the game. This result has been pieced together over about a decade by a number of authors
\citep{CBLuSt06,FR12,ABPS13,BFPRS14,LS18pm}.
Ironically, the `easy' games present the greatest challenge for algorithm design and analysis. 

The best known bound for an efficient algorithm for `easy' games is 
$\E[\Reg_n] \leq C(\Phi, \cL) \sqrt{n \log(n)}$,
where the constant $C(\Phi, \cL)$ can
be arbitrarily large, even for fixed $k$ and $d$ \citep{FR12,LS18pm}.
Furthermore, the algorithms achieving this bound are complicated to analyse and the proofs yield little insight into the structure of partial monitoring.
Recently we proved that for the `non-degenerate' (defined later) subset of easy games, the minimax regret is at most $\Reg_n^* \leq m k^{3/2} \sqrt{2n \log(k)}$ \citep{LS19pminfo}.
Unfortunately, however, our proof non-constructively appealed to minimax duality and the Bayesian regret analysis techniques by \cite{RV16}. No algorithm was provided, a
deficiency we now resolve.

Partial monitoring has been studied in a variety of contexts. For example, bandits with graph feedback \citep{ACDK15} and a linear feedback setting \citep{LAK14}.  
Some authors also consider a variant of the regret that refines the notion of optimality in hopeless games \citep{Rus99,MS03,Per11,MPS14}.
Our focus is on the adversarial setting, but the stochastic setup is also interesting and is better understood \citep{BSP11,VBK14,KHN15}.

\paragraph{Approach}
Our algorithms are based on exponential weights with importance-weighted loss difference estimators \citep{FS97}.
Crucially, the algorithms do not sample from the distribution proposed by exponential weights. Instead, they solve a convex optimisation
problem to find a loss difference estimator and new distribution over actions for which the loss cannot be much larger than the proposal distribution and 
the `stability' term in the bound of exponential weights is minimised. 
We then prove that the value of the optimisation problem appears in the resulting regret guarantee and provide upper bounds for different classes of games.
The most challenging aspect is to prove the existence of a suitable exploration distribution for locally observable non-degenerate games, which follows 
by combining a minimax theorem with insights from the Bayesian setting.
The idea to modify the distribution proposed by exponential weights is reminiscent of the work by \cite{MS09} for bandits with expert advice, though
the situation here is rather different.

\section{Notation and concepts}

We write $\zeros$ and $\ones$ for the column vectors of all zeros and all ones respectively.
For a positive semidefinite matrix $A$ and vector $x$, we let $\norm{x}_A^2 = x^\top A x$ and $\diag(x)$ be the diagonal matrix with $x$ on the diagonal.
We use $\norm{A}_\infty = \max_{ij} |A_{ij}|$ for the (entrywise) maximum norm of $A$, which we also use for the special case that $A$ is a vector. 
The minimum entry of a matrix is $\min(A) = \min_{ij} A_{ij}$.
The standard basis vectors are $e_1,\ldots,e_d$; we use the same symbols regardless of the dimension, which should be clear from the context in all cases.

\begin{wrapfigure}[8]{r}{4cm}
\vspace{-1cm}
\includegraphics[width=4cm]{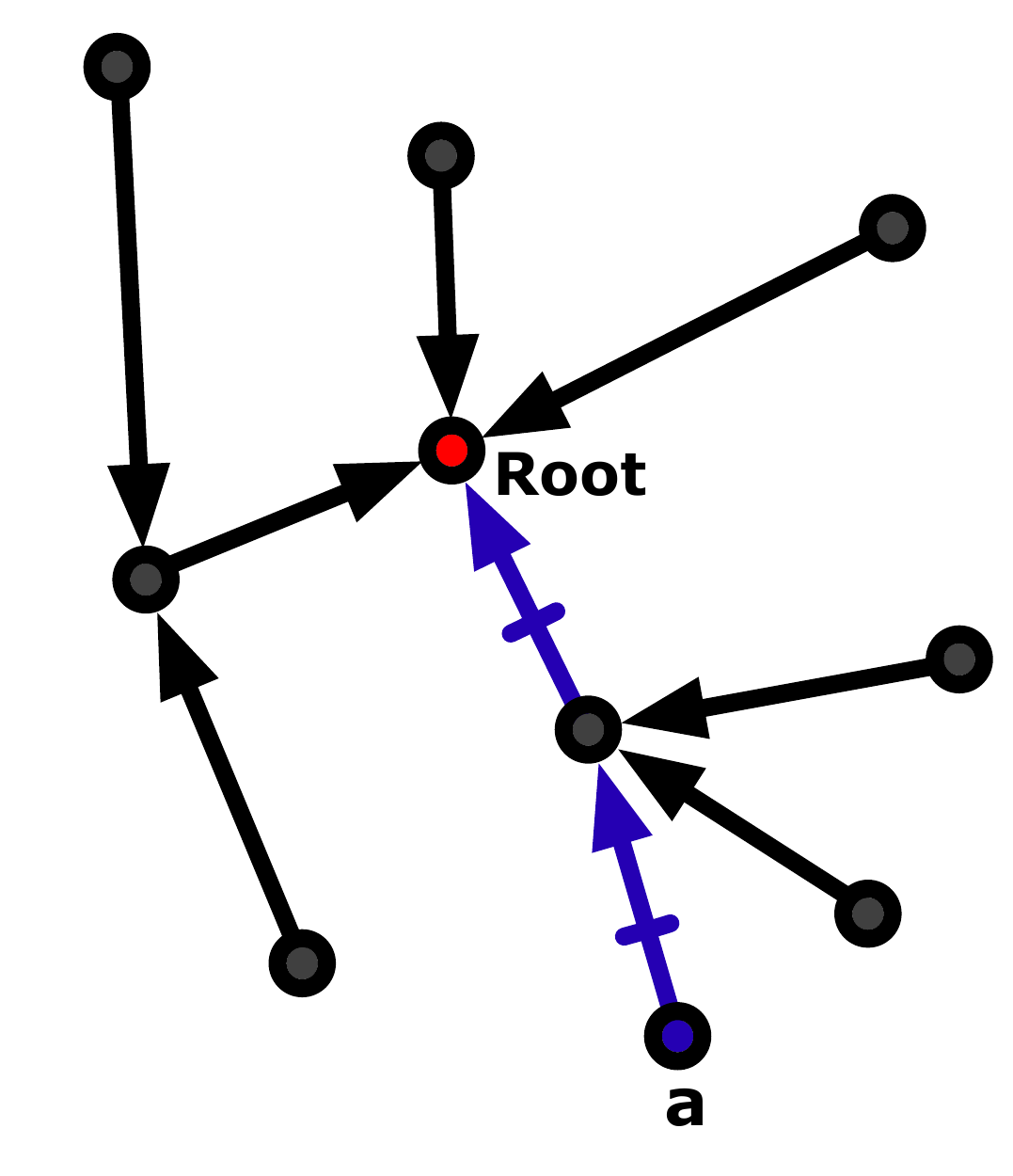}
\end{wrapfigure}
\paragraph{In-trees}
An in-tree with vertex set $[k]$ is a set $\cT \subseteq [k] \times [k]$ representing the edges of 
a directed tree with vertices $[k]$. Furthermore, we assume there is a root vertex denoted by $\troot_{\cT} \in [k]$ such 
that for all $a \in [k]$ there is a directed path $\tpath_{\cT}(a) \subseteq \cT$ from $a$ to the root. The path from the root is the empty set: $\tpath_\cT(\troot_{\cT}) = \emptyset$.
The figure depicts an in-tree over $k = 10$ vertices. The blue (barred) path is $\tpath_{\cT}(a)$.

\paragraph{Partial monitoring}
Throughout we fix a partial monitoring game $\cG = (\Phi, \cL)$ with loss matrix $\cL \in [0,1]^{k \times d}$ and signal matrix $\Phi \in \Sigma^{k \times d}$.
Let $\cD = \{\nu \in [0,1]^d : \norm{\nu}_1 = 1\}$ and $\cP = \{p \in [0,1]^k : \norm{p}_1 = 1\}$ be the probability simplices of dimension $d-1$ and $k-1$ respectively.
It is helpful to notice that if $p \in \cP$ and $\nu \in \cD$, then 
$p^\top \cL \nu$ is the expected loss suffered by a learner sampling an action from $p$ while the adversary samples its output from $\nu$.
Given an action $a \in [k]$ let $C_a = \{\nu \in \cD : e_a^\top \cL \nu \leq \min_{b \in [k]} e_b^\top \cL \nu \}$ be the set of probability vectors in $\cD$ 
where action $a$ is optimal to play in expectation if the adversary plays randomly according to $\nu$. We call $C_a$ the cell of action $a$. 
Cells are convex polytopes because they are bounded and are determined by finitely many non-strict linear constraints.
The collection $\{C_a : a \in [k]\}$, illustrated in \cref{fig:cell}, is called the cell decomposition.

\begin{figure}[h!]
\centering
\includegraphics[width=3.5cm]{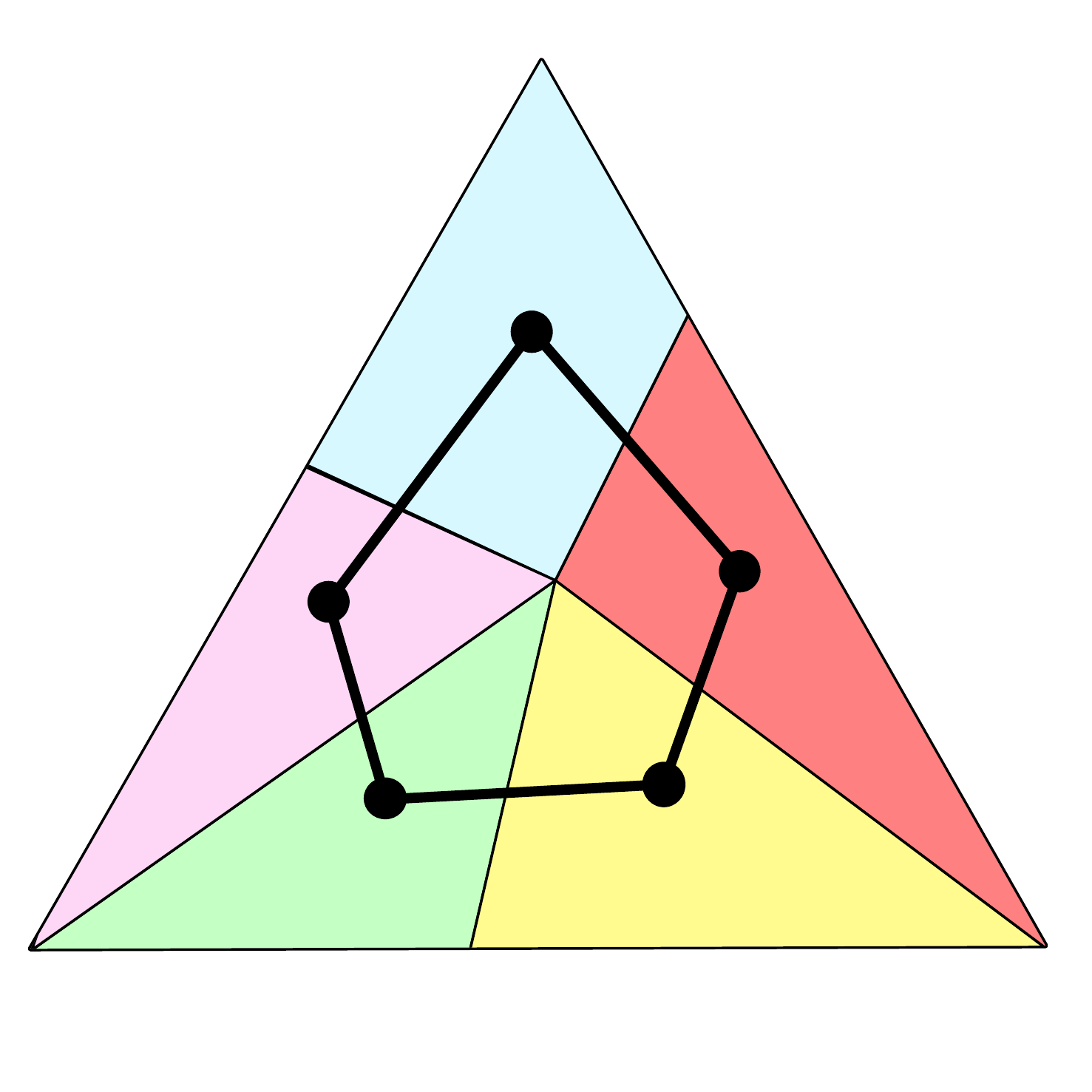}\hspace{2cm}
\includegraphics[width=3.5cm]{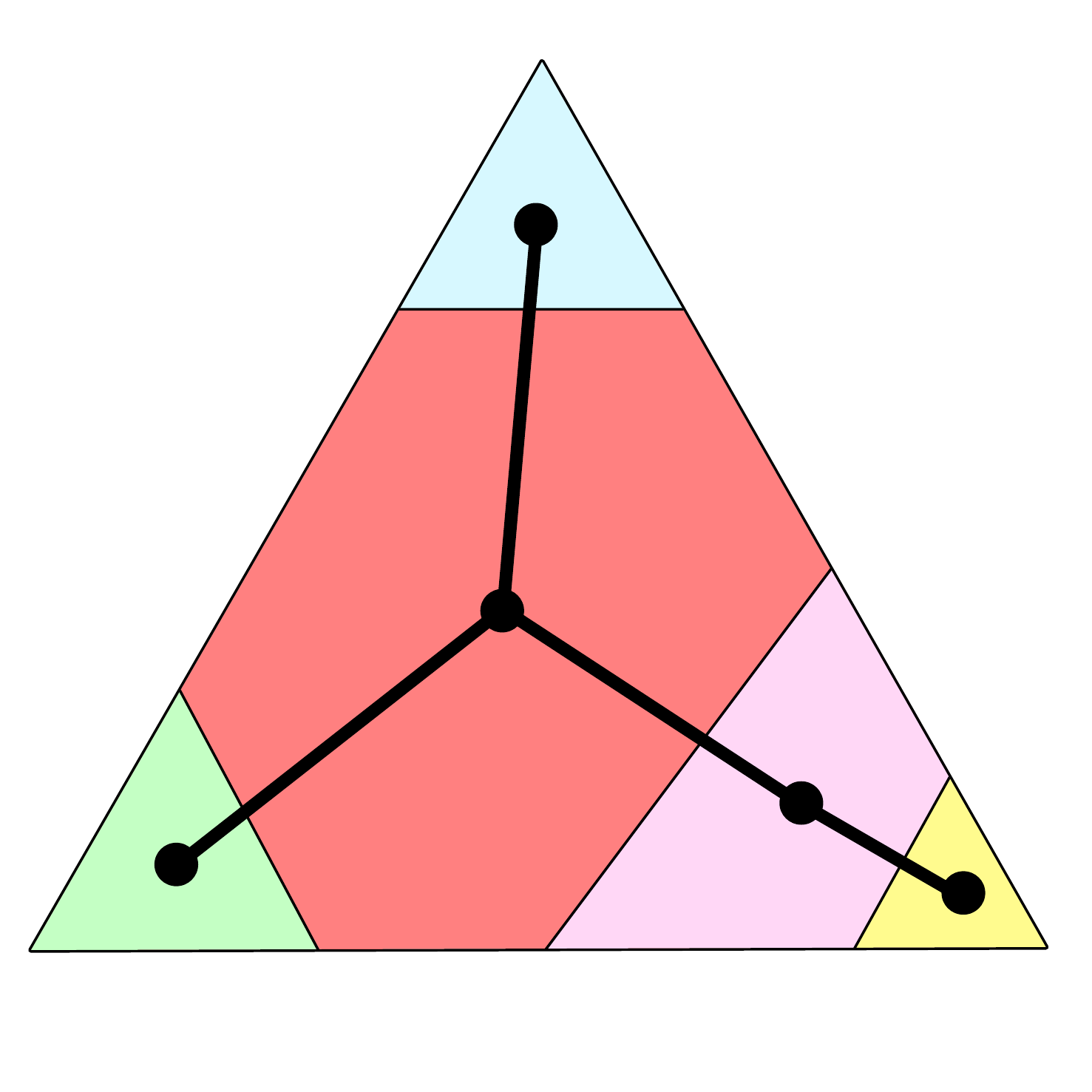}
\caption{Cell decompositions and neighbourhood graphs for two games with $d = 3$ and $k = 5$.}\label{fig:cell}
\end{figure}

\begin{remark}\label{rem:infinite}
A generalisation of the framework allows
$(x_t)_{t=1}^n$ to be chosen in an arbitrary outcome space $\cX$ and $\cL : [k] \times \cX \to [0,1]$ and $\Phi : [k] \times \cX \to \Sigma$ are arbitrary functions.
Our mathematical results continue to hold in this case with $d = |\cX|$, but the proposed algorithms may not be computationally efficient when $|\cX| = \infty$.
A short discussion of infinite games appears in \cref{sec:discussion}.
\end{remark}

\paragraph{Neighbourhood graph}

A key concept in partial monitoring is the neighbourhood relation, which gives those pairs of potentially optimal actions that can be optimal simultaneously. 
An action $a$ is called Pareto optimal if $\dim(C_a) = d-1$ where the dimension of a polytope is defined as the dimension of its affine hull as an affine subspace.
The set of Pareto optimal actions is denoted by $\Pi = \{a : \dim(C_a) = d-1\}$.
An action $a$ with $C_a \neq \emptyset$ and $\dim(C_a) \leq d-2$ is called degenerate while actions with $C_a = \emptyset$ are dominated.
Distinct actions $a$ and $b$ are duplicates if $(e_a - e_b)^\top \cL = \zeros$.
Pareto optimal actions $a$ and $b$ are neighbours if $\dim(C_a \cap C_b) = d-2$.
More informally, actions are neighbours if their cells share a boundary of dimension $d - 2$.
Note that $\dim(C_a \cap C_b) = d-1$ is only possible when $a$ and $b$ are duplicates.
The neighbourhood relation defines a graph over $[k]$. We let $\cE = \{(a, b) : a \text{ and } b \text{ are neighbours}\}$ be the set of edges in this graph. 
A game is called non-degenerate if it has no degenerate actions.
Of course, $\dim(\cD) = d-1$, so actions $a$ with $\dim(C_a) < d-1$ are optimal on a `negligible' subset of $\cD$, where they cannot be uniquely optimal.
For the remainder we make the following simplifying assumption.

\begin{assumption}
$\cG$ is globally observable, non-degenerate and contains no duplicate actions.
\end{assumption}

There is no particular reason to discard degenerate games except their analysis requires careful handling of certain edge cases, as we discuss briefly in the discussion
and extensively in other work \citep{LS18pm}. No modifications to the algorithm are required.

\newcommand{\neighbours}{\textrm{neighb}}
\paragraph{Observability}
The classification of a partial monitoring game depends on both the loss and signal matrices.
What is important to make a non-degenerate game `easy' is that the learner should have some way to estimate the loss differences between neighbouring actions 
by playing only those actions, a property known as local observability.
A game is globally observable if for all edges $e = (a, b) \in \cE$ in the neighbourhood graph there 
exists a function $\est_e : [k] \times \Sigma \to \R$ such that
\begin{align}
\cL_{ax} - \cL_{bx} = \sum_{c=1}^k \est_e(c, \Phi_{cx}) \text{ for all } x \in [d]\,.
\label{eq:est}
\end{align}
A non-degenerate game is locally observable if \cref{eq:est} holds and additionally $\est_e$ can be chosen so that $\est_e(c, \sigma) = 0$ for all $c \notin \{a, b\}$ and all $\sigma$.
Of course, all locally observable games are globally observable. 

\begin{remark}
The reader should be aware that for arbitrary (possibly degenerate) games the definition of local observability is that there exist estimation vectors such that
\cref{eq:est} holds and $\est_e(c, \sigma) = 0$ unless
$e_c^\top \cL = \alpha e_a^\top \cL + (1 - \alpha) e_b^\top \cL$ for some $\alpha \in [0,1]$.
For non-degenerate games the definitions are equivalent by \citep[Lemma 11]{BFPRS14}.
\end{remark}

The classification theorem we mentioned in the introduction says that 
\begin{align*}
\Reg_n^* = \begin{cases}
0 \,, & \text{if there exists an $a$ with $C_a = \cD$}\,; \\
\Theta(n^{1/2})\,, & \text{if the game is locally observable}\,; \\
\Theta(n^{2/3})\,, & \text{if the game is globally observable}\,; \\
\Omega(n)\,, & \text{otherwise}\,,
\end{cases}
\end{align*}
where the Big-Oh notation hides game-dependent constants.

\paragraph{Estimation}
The following lemma and discussion afterwards shows that for globally observable games 
\cref{eq:est} can be chained along paths in the neighbourhood graph to estimate the loss differences between any pair of actions, not just neighbours.
Let $\cH$ be the set of all functions $G : [k] \times \Sigma \to \R^k$.

\begin{lemma}\label{lem:est}
If $\cG$ is globally observable, then
there exists a function $G \in \cH$ such that for all $b, c \in \Pi$,
\begin{align*}
\sum_{a=1}^k \left(G(a, \Phi_{ax})_b - G(a, \Phi_{ax})_c\right) = \cL_{bx} - \cL_{cx} \,.
\end{align*}
\end{lemma}

\begin{proof}
Let $\cT\subseteq \cE$ be any in-tree over $\Pi$ and for $b \in \Pi$ let $G(a, \sigma)_b = \sum_{e \in \tpath_{\cT}(b)} \est_e(a, \sigma)$.
Then
\begin{align*}
\sum_{a=1}^k G(a, \Phi_{ax})_b
= \sum_{a=1}^k \sum_{e \in \tpath_{\cT}(b)} \est_e(a, \Phi_{ax})
= \cL_{bx} - \cL_{\troot_\cT x}\,.
\end{align*}
The result follows by repeating the argument for $c \in \Pi$ and taking the difference.
\end{proof}

Given a distribution $p \in \cP \cap (0,1)^k$ and $G$ satisfying the conclusion of \cref{lem:est}, it follows that if $A$ is sampled from $p$ and $x \in [d]$ is arbitrary, then
for actions $a, b$,
\begin{align}
&\E\left[\frac{(e_a - e_b)^\top G(A, \Phi_{Ax})}{p_A}\right]
= \sum_{c=1}^k (e_a - e_b)^\top G(c, \Phi_{cx})
= \cL_{ax} - \cL_{bx}\,.
\label{eq:lossest}
\end{align}
In other words, the function $G$ can be used with importance-weighting to estimate the loss differences.
The set of functions that satisfy the consequences of \cref{lem:est} are denoted by
\begin{align*}
\cHu = \set{G : (e_b - e_c)^\top \sum_{a=1}^k G(a, \Phi_{ax}) = \cL_{bx} - \cL_{cx} \text{ for all } b, c \in \Pi \text{ and } x \in [d]}\,.
\end{align*}

\paragraph{Bandit and full information games}
Bandit and full information games with finitely many possible losses can be modelled by finite partial monitoring games, and serve as useful examples.
Bandit games are those with $\cL = \Phi$ and full information games have $\Phi_{ax} = (\cL_{1x},\ldots,\cL_{kx})$.
Estimation functions witnessing the conclusion of \cref{lem:est} are easily constructed. The obvious choice for bandit games is $G(a, \sigma) = e_a \sigma$ while
for full information games $G(a, \sigma) = p_a \sigma$ where $p \in \cP$ is any probability distribution over the actions.

\paragraph{Exponential weights}

We briefly summarise a well-known bound on the regret of exponential weights.
For $q \in \cP$ define $\Psi_q : \R^k \to \R$ by
\begin{align}
\Psi_q(z) = \big\langle q, \exp(-z) + z - 1\big\rangle\,,
\label{def:psi}
\end{align}
where the exponential function is applied coordinate-wise.
Suppose that $(\hat y_t)_{t=1}^n$ is an arbitrary sequence of (loss) vectors with $\hat y_t \in \R^k$ and $(\eta_t)_{t=1}^n$ is a non-increasing sequence of positive learning rates.
Define a sequence of probability vectors $(q_t)_{t=1}^n$ by
\begin{align*}
q_{ta} = \frac{\exp\left(-\eta_t \sum_{s=1}^{t-1} \hat y_{sa}\right)}{\sum_{b=1}^k \exp\left(-\eta_t \sum_{s=1}^{t-1} \hat y_{sb}\right)}\,.
\end{align*}
Then the following bound on the regret holds for any $a^* \in [k]$ \citep[Chapter 28, for example]{LS19bandit-book},
\begin{align}
\sum_{t=1}^n \sum_{a=1}^k q_{ta} (\hat y_{ta} - \hat y_{ta^*}) \leq \frac{\log(k)}{\eta_n} + \sum_{t=1}^n \frac{\Psi_{q_t}\!\left(\eta_t \hat y_t\right)}{\eta_t}\,.
\label{thm:exp}
\end{align}
Note, there is no randomness here. The term involving $\Psi$ is sometimes called the stability term.
The following inequality is useful: 
\begin{align}
\Psi_q(\eta y) \leq 
\begin{cases}
\eta^2 \norm{y}_{\diag(q)}^2\,, & \text{if } \eta y \geq -\ones\,; \\
\frac{1}{2} \eta^2 \norm{y}^2_{\diag(q)}\,, & \text{if } \eta y \geq \zeros\,,
\end{cases}
\label{eq:psi}
\end{align}
which follows from the inequalities $\exp(-x) \leq x^2 - x + 1$ for all $x \geq -1$ and $\exp(-x) \leq x^2/2 - x + 1$ for $x \geq 0$.
We will use the fact that the perspective $(p, z) \mapsto p \Psi_q(z / p)$ is convex for $p > 0$.

\section{Exploration by optimisation}\label{sec:opt}

Our algorithm is a combination of exponential weights and a careful exploration strategy.
The following example game, called costly matching pennies, is helpful to gain some intuition:
\begin{align}
\cL =
\begin{pmatrix}
0 & 1 \\
1 & 0 \\
c & c
\end{pmatrix}
\qquad \text{ and } \qquad
\Phi =
\begin{pmatrix}
\bot & \bot \\
\bot & \bot \\
\textsc{h} & \textsc{t}
\end{pmatrix}\,.
\qquad
\begin{tikzpicture}[baseline=0cm]
\node[draw=none] at (0,0) {\includegraphics[width=4cm]{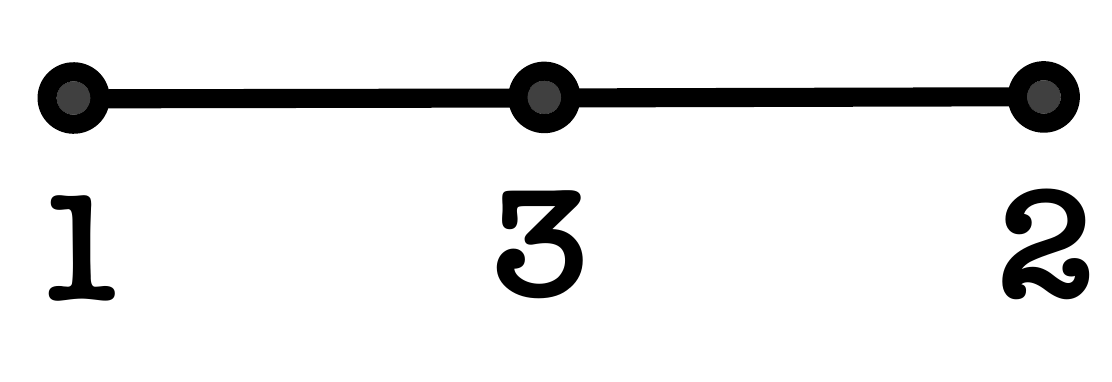}};
\end{tikzpicture}
\label{eq:example}
\end{align}
The figure on the right is the neighbourhood graph when $c = 1/4$, which shows the first two actions are separated by the third.
The structure of the feedback matrix means that the learner only gains information by playing the third action. 
Suppose that $q \in \cP$ is a distribution with $q_3$ close to zero and both $q_1$ and $q_2$ reasonably large.
Sampling an action from $q$ leads to a low probability of gaining information and a correspondingly high variance when estimating the difference between the losses of
the first and second actions.
Consider the transformation of $q$ defined by $p = q - \min(q_1, q_2) (e_1 + e_2) + 2 \min(q_1, q_2) e_3$, which is illustrated in \cref{fig:transform}. 
Then $p_3 \geq q_3$ and 
\begin{align}
(p - q)^\top \cL = -\frac{1}{2} \min(q_1, q_2) \ones\,.
\label{eq:flow-example}
\end{align}
Hence, any algorithm proposing to play distribution $q \in \cP$ with $\min(q_1, q_2) > 0$ could improve its decision by playing $p$, which decreases the 
expected loss and increases the amount of information. Our new algorithm solves an optimisation problem to find a sampling distribution and estimation function 
that minimise the sum of the loss relative to a distribution
proposed by exponential weights and the stability term in \cref{thm:exp}. In the example above the solution always results in a distribution $p$ with $\min(p_1, p_2) = 0$. 
By contrast, previous algorithms for adversarial locally observable partial monitoring games do not exhibit this behaviour \citep{FR12,LS18pm}.

\begin{figure}[h!]
\centering
\includegraphics[width=12cm]{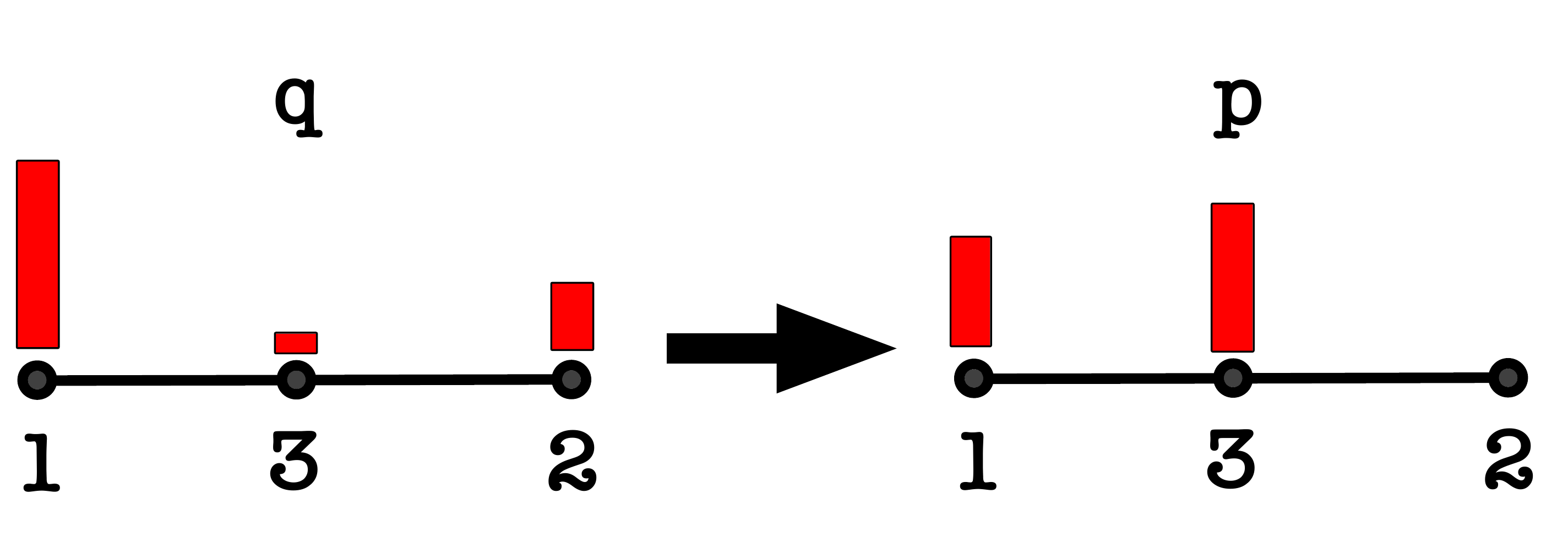}
\caption{An exploration distribution $p$ derived from $q$ for the game in \cref{eq:example}. The expected loss when playing $p$ is smaller than playing $q$ and simultaneously more information 
is gained because the third action is revealing.} \label{fig:transform}
\end{figure}

\paragraph{Optimisation problem}
Suppose that exponential weights proposes a distribution $q \in \cP$.
Our algorithm solves an optimisation problem to find an exploration distribution and estimation function that determine the loss estimators.
Given an estimation function $G \in \cH$ and outcome $x \in [d]$, define a `bias' function that measures the degree of bias when using importance-weighting
to estimate loss differences:
\begin{align*}
\bias_q(G ; x) = \left\langle q,\, \cL e_x - \sum_{a=1}^k G(a, \Phi_{ax})\right\rangle + \max_{c \in \Pi} \left(\sum_{a=1}^k G(a, \Phi_{ax})_c - \cL_{cx}\right)\,.
\end{align*}
As a function of $G$ the bias is max-affine and hence convex.
It is always non-negative and vanishes when the estimation function $G \in \cHu$ is unbiased.
For $q \in \cP$ and $\eta > 0$ let $\opt_q(\eta)$ be the value of the following convex optimisation problem:
\begin{mdframed}[roundcorner=1pt,backgroundcolor=black!5!white]
\begin{equation}
\label{eq:opt}
\begin{alignedat}{3}
&\underset{G \in \cH, p \in \cP}{\text{minimise}} \qquad && \max_{x \in [d]} \Bigg[
  \frac{(p-q)^\top \cL e_x + \bias_q(G ; x)}{\eta} + \frac{1}{\eta^2} \sum_{a=1}^k p_a \Psi_q\left(\frac{\eta G(a, \Phi_{ax})}{p_a}\right) 
\Bigg] \,. 
\end{alignedat}
\end{equation}
\end{mdframed}
We assume that \cref{eq:opt} can be solved exactly to obtain minimising values for $G \in \cH$ and $p \in \cP$.
Our algorithm, however, is robust to small perturbations of these quantities.
Numerical issues and a practical approximation are discussed in \cref{sec:approx}.
Let
\begin{align*}
\opt_*(\eta) = \sup_{q \in \cP} \opt_q(\eta)\,.
\end{align*}
Note that both $\opt_q(\eta)$ and $\opt_*(\eta)$ depend on $\cG$; this dependence is not shown to minimize clutter.
The optimisation problem can be formulated as an exponential cone problem and solved using off-the-shelf solvers. 
The algorithm is a simple combination of exponential weights using the exploration distribution and estimation function provided by solving \cref{eq:opt}.

\begin{algorithm}[h!]
\begin{simplealg}

\textbf{input:} $\eta$

\textbf{for} $t = 1,\ldots,n$: 

\algind Compute $\displaystyle Q_{ta} = \frac{\sind_\Pi(a) \exp\left(-\eta \sum_{s=1}^{t-1} \hat y_{sa}\right)}{\sum_{b \in \Pi} \exp\left(-\eta \sum_{s=1}^{t-1} \hat y_{sb}\right)}$ \\[0.1cm]

\algind Solve (\ref{eq:opt}) with $q = Q_t$ to find $P_t \in \cP$ and $G_t \in \cH$ \\[0.1cm]

\algind Sample $A_t \sim P_t$, observe $\sigma_t$ and compute $\displaystyle \hat y_t = \frac{G_t(A_t, \sigma_t)}{P_{tA_t}}$
\end{simplealg}
\caption{Exponential weights for partial monitoring with confidence}
\label{alg:simple}
\end{algorithm}

The regret of \cref{alg:simple} depends on the learning rate and the value of the optimisation problem, which depends on the structure of the game.
Bounds on $\opt_*(\eta)$ are provided subsequently.

\begin{theorem}\label{thm:main}
For any $\eta > 0$, the regret of \cref{alg:simple} is bounded by $\displaystyle \E[\Reg_n] \leq \frac{\log(k)}{\eta} + \eta n \opt_*(\eta)$.
\end{theorem}

\begin{proof}
Let $a^* = \argmin_{a \in [k]} \sum_{t=1}^n \cL(a, x_t)$ be the optimal action in hindsight, where ties are broken so that $a^* \in \Pi$ is Pareto optimal. 
Note, this is where we are using that the adversary is oblivious.
Then
\begin{align}
\E[\Reg_n] 
&= \E\left[\sum_{t=1}^n \sum_{b=1}^k P_{tb} (\cL_{b x_t} - \cL_{a^* x_t})\right] \nonumber \\
&\leq \E\left[\sum_{t=1}^n \left((P_t - Q_t)^\top \cL e_{x_t} + \bias_{Q_t}(G_t;x_t)\right)\right] + \E\left[\sum_{t=1}^n \sum_{b=1}^k Q_{tb} (\hat y_{tb} - \hat y_{ta^*})\right] \,.
\label{eq:reg1}
\end{align}
Next,
\begin{align}
\E\left[\sum_{t=1}^n \sum_{b=1}^k Q_{tb} (\hat y_{tb} - \hat y_{ta^*})\right] 
&\leq \frac{\log(k)}{\eta} + \frac{1}{\eta} \E\left[\sum_{t=1}^n \Psi_{Q_t}(\eta \hat y_t) \right] \label{eq:exp} \\
&= \frac{\log(k)}{\eta} + \frac{1}{\eta} \E\left[\sum_{t=1}^n \sum_{a=1}^k P_{ta} \Psi_{Q_t}\!\left(\frac{\eta G_t(a, \Phi_{ax_t})}{P_{ta}}\right) \right]\,, \nonumber
\end{align}
where \cref{eq:exp} follows from \cref{thm:exp} and the definitions of $Q_t$ and $\hat y_t$.
The result follows by combining \cref{eq:reg1} and the definition of $\opt_*(\eta)$.
\end{proof}

\paragraph{Applications}
\cref{tab:bounds} provides bounds on $\opt_*(\eta)$ for different games and the regret bound that results from optimising the learning rate.
The proofs are provided in \cref{sec:bounds,sec:bound-local}. 
Except for locally observable games, they mirror existing proofs bounding the stability of exponential weights.
In this way many other results could be added to this table, including bandits with graph feedback \citep{ACDK15} and linear bandits with finitely many arms \citep{BCK12}.

\begin{table}[h!]
\centering
\renewcommand{\arraystretch}{1.9}
\small
\scalebox{0.9}{
\begin{tabular}{|p{4.3cm}|p{1.8cm}p{3.2cm}p{1.4cm}p{4.3cm}|}
\hline 
\rowcolor{black!10!white}
 \textbf{Game type} &
{ \renewcommand{\arraystretch}{0.95}
 \begin{tabular}{c}
 $\text{\textbf{opt}}^{\bm *}\bm{ (\eta)\, }$ \\
  \textbf{bound}
\end{tabular}  
}
   & \textbf{Conditions} & \textbf{Ref.} & \textbf{Regret}  \\ \hline
\cellcolor{black!10!white} \textsc{Bandit} & $k/2$ & & Prop.~\ref{prop:bandit} & $\sqrt{2nk \log(k)}$\\
\cellcolor{black!10!white} \textsc{Full information} & $1/2$ & & Prop.~\ref{prop:full} & $\sqrt{2n \log(k)}$ \\ 
\cellcolor{black!10!white} \textsc{Globally observable} & $\displaystyle c_{\cG} / \sqrt{\eta}$ & $\displaystyle \eta \leq 1/c_{\cG}^2$ & Prop.~\ref{prop:hard} & $3(c_{\cG}n/2)^{2/3}(\log(k))^{1/3}$ \\ 
\cellcolor{black!10!white} \textsc{Locally observable \newline non-degenerate} & $3k^3m^2$ & $\displaystyle \eta \leq 1/(m k^2)$ & Prop.~\ref{prop:easy} & $2 k^{3/2}m \sqrt{3n \log(k)}$ \\ \hline
\end{tabular}
}
\caption{Upper bounds on $\opt_*(\eta)$ and the regret of \cref{alg:simple} for different games. 
The constant $c_{\cG}$ is game-dependent and can be exponentially large in $d$, which we believe is unavoidable.}\label{tab:bounds}
\end{table}

\section{Online learning rate tuning}\label{sec:adaptive}

Tuning the learning rate used by \cref{alg:simple} is delicate.
First, it is not clear that $\opt_*(\eta)$ can be computed efficiently in general.
Second, the learning rate that minimises the bound in \cref{thm:main} may be overly conservative.
\cref{alg:adaptive} mitigates these issues by using an adaptive learning rate.
The algorithm is parameterised by a constant $B$ that determines the initialisation of the learning rate. 
$B$ should be chosen large enough that $\eta = 1/B$ satisfies the conditions for the relevant game in \cref{tab:bounds}, but the additional regret
from choosing $B$ too large is only additive.

\begin{algorithm}[h!]
\begin{simplealg}
\textbf{input:} $B$

\textbf{for} $t = 1,2,\ldots,$: 

\algind Set $\displaystyle \eta_t = \min\left\{\frac{1}{B},\,\, \sqrt{\frac{\log(k)}{1 + \sum_{s=1}^{t-1} V_s}}\right\}$ \\[0.1cm]

\algind Compute $\displaystyle Q_{ta} = \frac{\sind_\Pi(a) \exp\left(-\eta_t \sum_{s=1}^{t-1} \hat y_{sa}\right)}{\sum_{b \in \Pi}\exp\left(-\eta_t \sum_{s=1}^{t-1} \hat y_{sb}\right)}$ \\[0.1cm]

\algind Solve (\ref{eq:opt}) with $\eta = \eta_t$ and $q = Q_t$ to find $V_t = \max\{0, \opt_{Q_t}(\eta_t)\}$ and corresponding $P_t$ and $G_t$ \\[0.1cm]

\algind Sample $A_t \sim P_t$, observe $\sigma_t$ and compute $\displaystyle \hat y_t = \frac{G_t(A_t, \sigma_t)}{P_{tA_t}}$
\end{simplealg}
\caption{Adaptive exponential weights for partial monitoring}
\label{alg:adaptive}
\end{algorithm}
We now present a general theorem that bounds the regret as a function of $(V_t)_t$, which is computed by the algorithm. 
This theorem implies that  
\cref{alg:adaptive} recovers all regret bounds in \cref{tab:bounds} up to small constant factors and additive terms. 

\begin{theorem}\label{thm:adaptive}
There exists a universal constant $c > 0$ such that
the regret of \cref{alg:adaptive} is bounded by 
\begin{align*}
\E[\Reg_n] \leq 5\E\left[\sqrt{\left(1 + \sum_{t=1}^n V_t\right) \log(k)}\right] + \E\left[\max_{t \in [n]} V_t \right] \sqrt{\log(k)} + B \log(k) \,. 
\end{align*}
\end{theorem}

A corollary using the definition of $V_t$ is that the regret of \cref{alg:adaptive} is bounded by
\begin{align}
\E[\Reg_n] = O\left(\sqrt{n \sup \{ \max\{0, \opt_*(\eta)\} : \eta \leq 1/B\} \log(k)}+B\log(k)\right)\,.
\label{eq:sup-bound}
\end{align}
This bound is most useful for  full information, bandit and locally observable 
non-degenerate games when $B$ can be chosen so that $\eta_1 \leq 1/B$ satisfies the conditions in the second column of \cref{tab:bounds}. As a consequence, for games of this category
\cref{thm:adaptive} recovers the bounds in the last column \cref{tab:bounds} up to small constant factors and additive terms.

For games that are globally observable but not locally observable $\opt_*(\eta) \to \infty$ as $\eta \to 0$ and the supremum in \cref{eq:sup-bound} is infinite.
Soon, we will argue that the learning rate used by \cref{alg:adaptive}  does not decrease too fast and that the algorithm still achieves the regret bound shown in \cref{tab:bounds} for globally observable games.

\begin{proof}[Proof sketch of \cref{thm:adaptive}]
We explain only the differences relative to the proof of \cref{thm:main}.
Recall that $V_t = \max\{0, \opt_{Q_t}(\eta_t)\}$.
Note, the learning rate $\eta_t$ is non-increasing.
Hence, by \cref{thm:exp},
\begin{align}
\E[\Reg_n] 
&\leq \E\left[\sum_{t=1}^n \sum_{a=1}^k Q_{ta}(\hat y_{ta} - \hat y_{ta^*}) + \sum_{t=1}^n (P_t - Q_t)^\top \cL e_{x_t} + \bias_{Q_t}(G_t ; x_t)\right] \nonumber \\
&\leq \E\left[\frac{\log(k)}{\eta_n} + \sum_{t=1}^n \frac{\Psi_{Q_t}(\eta_t \hat y_t)}{\eta_t} + \sum_{t=1}^n (P_t - Q_t)^\top \cL e_{x_t} + \bias_{Q_t}(G_t ;x_t)\right] \label{eq:exp-change} \\
&\leq \E\left[\frac{\log(k)}{\eta_n} + \sum_{t=1}^n \eta_t V_t\right] \label{eq:ad-opt}\,,
\end{align}
where \cref{eq:ad-opt} follows from the same argument as the proof of \cref{thm:main} and the definition of $V_t$. 
The second term is bounded using \cref{lem:tech1} in the appendix by
\begin{align*}
\sum_{t=1}^n \eta_t V_t
&\leq 4 \sqrt{\left(1 + \frac{1}{2} \sum_{t=1}^n V_t\right) \log(k)} + 
\max_{t \in [k]} V_t
\sqrt{\log(k)}\,.
\end{align*}
The definition of $(\eta_t)_{t=1}^n$ means that
\begin{align*}
\frac{\log(k)}{\eta_n} \leq B \log(k) + \sqrt{\left(1 + \sum_{t=1}^n V_t\right) \log(k)}\,.
\end{align*}
The bound follows by combining the parts and naive algebra.
\end{proof}

As promised, we now show that for sufficiently large $B$ the algorithm achieves the best known regret for any globally observable game.

\begin{proposition}
Fix a globally observable game $\cG$.
Suppose that $\alpha > 0$ and
$\opt_*(\eta) \leq \alpha / \sqrt{\eta}$ for all $\eta \leq 1/B$. 
Then, the regret of \cref{alg:adaptive} on $\cG$ is at most
\begin{align*}
\E[\Reg_n] = O\left((n \alpha)^{2/3} (\log(k))^{1/3} + B \log(k) \right)\,,
\end{align*}
where the Big-Oh hides only universal constants.
\end{proposition}
Note that the conditions of this result will be satisfied with $\alpha = c_{\cG}$ once $B\ge c_{\cG}^2$ with (cf. \cref{tab:bounds}).
\begin{proof}
The result follows from \cref{thm:adaptive} and an almost sure bound on $\sum_{t=1}^n V_t$.
Clearly, $\eta_t \leq 1/B$ and so by assumption $V_t = \max\{0, \opt_{Q_t}(\eta_t)\} \leq \alpha / \sqrt{\eta_t}$.
Then, using the definition of $(\eta_t)_{t=1}^n$,
\begin{align*}
\frac{\log(k)}{\eta_{t+1}^2} \leq \frac{\log(k)}{\eta_t^2} + \frac{\alpha}{\eta_t^{1/2}} 
= \frac{\log(k)}{\eta_t^2} + \frac{\alpha}{\log(k)^{1/4}} \left(\frac{\log(k)}{\eta_t^2}\right)^{\frac{1}{4}}\,.
\end{align*}
Hence, using the definition of $\eta_n$ and \cref{lem:tech2} in the appendix, 
\begin{align*}
1 + \sum_{t=1}^n V_t \leq \frac{\log(k)}{\eta_n^2} \leq \left(\frac{3\alpha (n-1)}{4\log(k)^{1/4}} + \max\{1, B^2 \log(k)\}^{3/4}\right)^{4/3}\,.
\end{align*}
Substituting the above bound into the dominant term of \cref{thm:adaptive} shows that 
\begin{align*}
\sqrt{\left(1 + \sum_{t=1}^n V_t\right) \log(k)} = O\Big((k n \alpha)^{2/3} (\log(k))^{1/3} + B \log(k)\Big)\,.
\end{align*}
The result is completed by noting that $\max_{t \in [n]} V_t \leq \alpha \eta_n^{-1/2}$ is lower-order. 
\end{proof}

\section{Bandit, full information and globally observable games}\label{sec:bounds}

We now bound $\opt_*(\eta)$ for bandit, full information and globally observable games.
All results follow from the usual arguments for bounding the stability term in the regret guarantee
for exponential weights in \cref{thm:exp}.

\begin{proposition}\label{prop:bandit}
For bandit games, $\opt_*(\eta) \leq k/2$ for all $\eta > 0$.
\end{proposition}

\begin{proof}
Let $q \in \cP$ be arbitrary and let $p = q$ and $G(a, \sigma) = e_a \sigma$.
This corresponds to the usual importance-weighted estimators used for $k$-armed bandits.
Then
\begin{align*}
\opt_q(\eta) 
&\leq \max_{x \in [d]} \sum_{a=1}^k \frac{p_a}{\eta^2} \Psi_q\left(\frac{\eta G(a, \Phi_{ax})}{p_a}\right) 
= \max_{x \in [d]} \sum_{a=1}^k \frac{p_a}{\eta^2} \Psi_q\left(\frac{\eta \cL_{ax} e_a}{p_a}\right) 
\leq \max_{x \in [d]} \sum_{a=1}^k \frac{\cL_{ax}^2}{2}
\leq \frac{k}{2}\,, 
\end{align*}
where in the second last inequality we used \cref{eq:psi} and the fact that $p_a = q_a$.
\end{proof}

\begin{proposition}\label{prop:full}
For full information games, $\opt_*(\eta) \leq 1/2$ for all $\eta > 0$.
\end{proposition}

\begin{proof}
As in the previous proof let $p = q$, but now choose $G(a, \sigma) = p_a \sigma$, which is unbiased.
The argument then follows along the same lines as the proof of \cref{prop:bandit}.
\end{proof}

\begin{remark}
You might wonder whether these choices of $p$ and $G$ actually minimise \cref{eq:opt}, in which case the algorithm would reduce to Hedge
for full information games. As we show in \cref{app:hedge}, however, they do not. The minimisers of \cref{eq:opt} shift the loss estimates and play a distribution
that is close to $q$, but not exactly the same. A similar story holds for bandits.
\end{remark}

\begin{proposition}\label{prop:hard}
For non-degenerate globally observable games
there exists a constant $c_{\cG}$ depending only on $\Phi$ and $\cL$ such that 
for all $\eta \leq 1/c_{\cG}^2$,
\begin{align*}
\opt_*(\eta) \leq \frac{c_{\cG}}{\sqrt{\eta}}\,. 
\end{align*}
\end{proposition}

\begin{proof}
By the definition of a globally observable game there exists an unbiased estimation function $G \in \cHu$.
Let $\beta = \norm{G}_\infty$ and $c_{\cG} = \max\{1, 2k\beta\}$.
Then let $\gamma = k \beta \sqrt{\eta}$ and $p = (1 - \gamma) q + \gamma \ones / k$, which
is a probability distribution since $\gamma \in [0,1]$ for $\eta \leq 1/c_{\cG}^2$.
We claim that $\eta G(a, \Phi_{ax}) / p_a \geq -\ones$, which follows from the definitions of $\gamma$ and $\beta$ so that
\begin{align*}
p_a \ones \geq \frac{\gamma}{k} \ones = \beta \sqrt{\eta} \ones \geq \beta \eta \ones \geq \eta G(a, \Phi_{ax}) \,,
\end{align*}
where the second inequality uses the fact that $c_{\cG} \geq 1$ and $\eta \leq 1/c_{\cG}^2 \leq 1$.
To bound the objective notice that for any $x \in [d]$ it holds that 
\begin{align*}
\frac{1}{\eta} (p - q)^\top \cL e_x = \frac{\gamma}{\eta} (\ones / k - q)^\top \cL e_x\leq \frac{\gamma}{\eta} = \frac{k \beta}{\sqrt{\eta}}\,.
\end{align*}
For the second term in the objective, by \cref{eq:psi},
\begin{align*}
\frac{1}{\eta^2} \max_{x \in [d]} \sum_{a=1}^k p_a \Psi_q\left(\frac{\eta G(a, \Phi_{ax})}{p_a}\right) 
&\leq \max_{x \in [d]} \sum_{a=1}^k \frac{\norm{G(a, \Phi_{ax})}_{\diag(q)}^2}{p_a} \\
&\leq \frac{k}{\gamma} \max_{x \in [d]} \sum_{a=1}^k \ip{q, G(a, \Phi_{ax})^2} \\
&\leq \frac{k^2 \beta^2}{\gamma} 
= \frac{k \beta}{\sqrt{\eta}}\,. 
\end{align*}
The result follows by combining the previous two displays and the definition of $c_{\cG}$.
\end{proof}

\section{Locally observable games}\label{sec:bound-local}

Controlling $\opt_*(\eta)$ for locally observable games is more involved.
The main result of this section is a proof of the following proposition.

\begin{proposition}\label{prop:easy}
For locally observable non-degenerate games and $\eta \leq 1/(2m k^2)$,
\begin{align*}
\opt_*(\eta) \leq 3 m^2 k^3\,.
\end{align*}
\end{proposition}

We make use of the water transfer operator, which is a
construction from our earlier paper that provides an exploration distribution suitable for locally observable games in the Bayesian setting.
The challenge in partial monitoring is that the observability structure only allows for pairwise comparison between neighbours. 
This is problematic when two non-neighbouring actions are played with high probability and the actions separating them are played with low probability.
Given distributions $q \in \cP$ and $\nu \in \cD$, the water transfer operator `flows' probability in $q$ towards the greedy action $a$ for which $\nu \in C_a$.
Then all loss differences can be estimated relative to the greedy action.
This decreases the variance of estimation without increasing the expected loss when the adversary samples its action from $\nu$. 

\begin{lemma}[\citealt{LS19pminfo}]\label{lem:water}
Suppose that $\cG$ is non-degenerate and locally observable and
$\nu \in \cD$. Then there exists a function $W_\nu : \cP \to \cP$ such that the following hold for all $q \in \cP$:
\begin{enumerate}
\item[(a)] The expected loss does not increase: $\left(W_\nu(q) - q\right)^\top \cL \nu \leq 0$.
\item[(b)] Action probabilities are not too small: $W_\nu(q)_a \geq q_a/k$ for all $a \in [k]$.
\item[(c)] Probabilities increase towards the root of some in-tree: there exists an in-tree $\cT \subseteq \cE$ over $[k]$ such that $W_\nu(q)_a \leq W_\nu(q)_b$ for all $(a, b) \in \cT$.
\end{enumerate}
\end{lemma}

A simplified proof of the above lemma is provided for completeness in \cref{app:water}.

\begin{proof}[Proof of \cref{prop:easy}]
Let $q \in \cP$.
By Sion's minimax theorem
\begin{align*}
\opt_q(\eta)
&\leq \min_{G \in \cHu, p \in \cP} \max_{\nu \in \cD} \left[\frac{1}{\eta} (p - q)^\top \cL \nu + \frac{1}{\eta^2} \sum_{x=1}^d \nu_x \sum_{a=1}^k p_a \Psi_q\left(\frac{\eta G(a, \Phi_{ax})}{p_a}\right)\right] \\
&= \max_{\nu \in \cD} \min_{G \in \cHu, p \in \cP} \left[\frac{1}{\eta} (p - q)^\top \cL \nu + \frac{1}{\eta^2} \sum_{x=1}^d \nu_x \sum_{a=1}^k p_a \Psi_q\left(\frac{\eta G(a, \Phi_{ax})}{p_a}\right)\right] \,,
\end{align*}
where in the first inequality we added the constraint that $G \in \cHu$, which zeros the bias term. 
Let $\nu \in \cD$ and let $\cT$ and $r = W_\nu(q)$ be the in-tree over $[k]$ and distribution in $\cP$ provided by the water transfer operator (\cref{lem:water}), respectively.
Define $G \in \cHu$ by
\begin{align*}
G(a, \sigma)_b = \sum_{e \in \tpath_\cT(b)} \est_e(a, \sigma) \,.
\end{align*}
By \cref{lem:bounds} and the assumption that $\cG$ is non-degenerate, $\est_e$ can be chosen so that $\norm{\est_e}_\infty \leq m/2$.
Since paths in $\cT$ have length at most $k$ it follows that
\begin{align*}
\norm{G}_\infty \leq k m / 2\,.
\end{align*}
Furthermore, $G \in \cHu$ by the proof of \cref{lem:est}.
Then let $\gamma = \eta m k^2 / 2$ and $p = (1 - \gamma) r + \gamma \ones / k$, which means that for any $x \in [d]$, 
\begin{align*}
\frac{\eta G(a, \Phi_{ax})}{p_a} \geq -\frac{\eta m k^2}{2\gamma} = -1\,.
\end{align*}
Additionally, the assumption that $\eta \leq 1/(mk^2)$ means that $\gamma \leq 1/2$ so that $r \geq p/2$.
Hence, by \cref{eq:psi} and using Parts (b) and (c) of \cref{lem:water} with the definition of $r$,
\begin{align*}
\frac{1}{\eta^2} \sum_{a=1}^k p_a \Psi_q\left(\frac{\eta G(a, \Phi_{ax})}{p_a}\right)
&\leq\sum_{a=1}^k \frac{\norm{G(a, \Phi_{ax})}^2_{\diag(q)}}{p_a} \\
&\leq 2\sum_{a=1}^k \frac{\norm{G(a, \Phi_{ax})}^2_{\diag(q)}}{r_a} \\
&= 2\sum_{b=1}^k \sum_{a=1}^k \frac{q_b}{r_a} \left(\sum_{e \in \tpath_{\cT}(b)} \est_e(a, \Phi_{ax}) \right)^2 \\
&\leq \frac{m^2}{2}\sum_{b=1}^k \sum_{a=1}^k \frac{q_b}{r_a} \left(\sum_{e  \in \tpath_{\cT}(b)} \one{a \in e}\right)^2 \\
&\leq 2k^3m^2\,,
\end{align*}
where we used Part (b) of \cref{lem:water} to show that $q_b \leq k r_b$ and Part (c) to show that $r_a \geq r_b$ for $a \in \tpath_{\cT}(b)$.
Finally,
\begin{align*}
\frac{1}{\eta} (p - q)^\top \cL \nu 
&= \frac{1}{\eta}(r - q)^\top \cL \nu + \frac{\gamma}{\eta} (\ones / k - r)^\top \cL \nu
\leq \frac{\gamma}{\eta} (\ones / k - r)^\top \cL \nu
\leq \frac{\gamma}{\gamma} = m k^2 \leq k^3 m^2\,.
\end{align*}
Hence $\opt_q(\eta) \leq 3 k^3 m^2$.
\end{proof}

\begin{remark}
The bound can be improved to $\opt_q(\eta) \leq 3 k m^2 \diam(\cE)^2$, where $\diam(\cE)$ is the diameter of the neighbourhood graph.
\end{remark}

\section{Discussion}\label{sec:discussion}

We introduced a new algorithm for finite partial monitoring that is efficient, nearly parameter free and enjoys roughly the 
best known regret in all classes of games. Notably, this is the first efficient algorithm for which the regret is independent of arbitrarily large game-dependent
constants for locally observable non-degenerate games.
A natural criticism of previous algorithms for partial monitoring is that the algorithms are generally quite conservative and not practical for normal problems.
As far as we can tell, the proposed algorithm does not suffer from this problem, at least recovering standard bounds in bandit and full information settings. 
In certain cases the algorithm may also adapt to the choices of the adversary.
The principle for finding an exploration distribution and estimation procedure is generic and may work well in other problems.

\paragraph{Lower bounds}
The best known lower bound for locally observable partial monitoring games is either $\Omega(\sqrt{kn})$ or $\Omega(d\sqrt{n})$, 
which are witnessed by a standard Bernoulli bandit \citep{ACFS95} and a result by the authors \citep{LS18pm}.
If pressed, we would speculate that $\Theta(d \sqrt{kn})$ is the correct worst-case regret over all $d$-outcome $k$-action non-degenerate locally
observable partial monitoring games, at least as $n$ tends to infinity. 

\paragraph{High probability bounds}
By replacing the bias term in \cref{eq:opt} with a constraint on a certain moment-generating function the algorithm can be adapted to prove
high probability bounds. Details are provided in \cref{sec:hp}.

\paragraph{Infinite outcome spaces}
Finiteness of the outcome space was not used in the proofs of \cref{thm:main} or \cref{thm:adaptive} and in particular the results in \cref{tab:bounds} continue to hold in this case 
The main cost of infinite outcome spaces is that the optimisation problem \cref{eq:opt} is unlikely to be tractable without additional structure.
Classic examples of infinite games for which the regret can be well controlled are bandit and full information games.
In both games the outcomes $(x_t)_{t=1}^n$ are chosen in $\cX = [0,1]^k$ and $\cL(a, x) = x_a$ (using the notation of \cref{rem:infinite}).
The signal function is $\Phi(a,x) = x_a$ for bandits and $\Phi(a,x) = x$ for the full information games.
Exploring the existence of a simple classification theorem for infinite-outcome games is an interesting future direction.
Understanding when \cref{eq:opt} is tractable is also intriguing.

\paragraph{Game-dependent bounds}
One of the objectives of this work was to design an efficient algorithm for which the regret does not depend on arbitrarily large game-dependent constants.
Naturally it is desirable to have small game-dependent constants and adaptivity to the choices of the adversary. 
\cref{tab:bounds} provides upper bounds on $\opt_*(\eta)$ for various classes, but the actual values depends on the game.
Understanding the dependence of this optimisation problem on the structure of the loss and signal matrices is an interesting open direction.
Also interesting is whether or not $\opt_*(\eta)$ is a fundamental quantity for the difficulty of the game and/or the regret of our algorithms.

\paragraph{Adaptivity}
\cref{alg:adaptive} already exhibits some adaptivity in the lucky situation that $V_t$ is small. This is not entirely satisfactory, however,
since $V_t$ is a random variable that depends on the choices of both the learner and the adversary.
We anticipate that all the usual enhancements for adaptivity -- log barrier, biased estimates and optimism -- can be applied here \citep[for example]{RS13,BCL17,CL18,BLL19}.
A related challenge would be to seek a best-of-both-worlds result, perhaps using the INF potential \citep{ZLW19}.

\paragraph{Beyond exponential weights}
The objective in \cref{eq:opt} is chosen so that the terms in \cref{eq:exp} are well controlled, which corresponds to bounding the stability term in the regret analysis of exponential weights. 
Other algorithms can be obtained by replacing exponential weights with follow the regularized leader and Legendre potential $F$. A standard regret bound (holding under certain
technical conditions) is
\begin{align}
\label{eq:ftrl}
\E[\Reg_n] 
&\leq \frac{\diam_F(\cP)}{\eta} + \frac{1}{\eta} \E\left[\sum_{t=1}^n \sum_{a=1}^k P_{ta} D_{F^*}\left(\nabla F(Q_t) - \frac{\eta G_t(a, \Phi_{ax_t})}{P_{ta}}, \nabla F(Q_t)\right)\right] \\
&\qquad\qquad\qquad\qquad\qquad + \E\left[\sum_{t=1}^n (P_t - Q_t)^\top \cL e_{x_t} + \bias_{Q_t}(G_t ; x_t)\right]\,. \nonumber
\end{align}
where $\diam_F(\cP) = \max_{x,y \in \cP} F(x) - F(y)$ is the diameter and $D_{F^*}(x, y)$ is the Bregman divergence between $x$ and $y$ with respect to the Fenchel conjugate of $F$. 
Let
\begin{align*}
\Psi_q(z) = D_{F^*}(\nabla F(q) - z, \nabla F(q))\,.
\end{align*}
Then convexity of $F^*$ implies that the perspective $(p, z) \mapsto p \Psi_q(z / p)$ is also convex for $p > 0$. 
When $F$ is the unnormalised negentropy, the definition above reduces to \cref{def:psi}.
All this means that the same approach holds more broadly for other potentials, which carry certain advantages in some settings \citep[and others]{AB09,BCL17,CL18,BLL19}.
For more details on follow the regularised leader and bounds of the form in \cref{eq:ftrl}, see \citep[Chapter 28]{LS19bandit-book} and \citep{Haz16}. 
We leave a deeper exploration of these ideas for the future.

\paragraph{Degenerate games}
The non-degeneracy assumption is purely for simplicity.
Only the proof of \cref{prop:easy} and its dependents need to be modified in minor ways. The notable difference is that the magnitude of the estimation vectors
is no longer guaranteed to be small. More specifically,
\cref{lem:bounds} does not hold when estimating loss differences between actions $(a, b)$ for which there are degenerate actions $c$ with $C_c = C_a \cap C_b$.
As in our previous work, using \cref{lem:bounds-global} instead introduces constants that may be exponential in $d$, which we believe is unavoidable \citep{LS19pminfo}.
Duplicate actions can be handled similarly and have the same affect.

\paragraph{Connections between stability and the information ratio}
\cite{ZL19} have shown that the generalised information ratio can be bounded by a worst-case bound on the stability term of mirror descent, which
makes a connection between the information-theoretic tools and those from online convex optimisation. Here we work in the other direction, using duality and the techniques
for bounding the information ratio to bound the stability term.  
The argument does not provide an equivalence between stability and the information ratio, but perhaps reinforces the feeling that there is an interesting connection here.

\paragraph{Acknowledgements}
We are grateful to Andr\'as Gy\"orgy for many insightful comments.

\appendix

\bibliographystyle{apalike}
\bibliography{pm-simple}

\section{High probability bounds}\label{sec:hp}

The same design principle can be used to construct algorithms for which the regret is controlled with high probability.
The idea is to replace the bias term in the objective with constraints on the range
of the loss estimators and on an appropriately chosen moment-generating function.
Given $q \in \cP$ and $\eta > 0$, let $\opthp_q(\eta)$ be the solution to the following optimisation problem:

\begin{mdframed}[roundcorner=1pt,backgroundcolor=black!5!white]
\begin{equation}
\label{eq:opt-hp}
\begin{alignedat}{3}
&\underset{G \in \cH, p \in \cP,\, \lambda \geq 0}{\text{minimise}} \qquad && \lambda + \frac{2}{\eta^2}\max_{x \in [d]}
  \sum_{a=1}^k p_a \Psi_q\left(\frac{\eta G(a, \Phi_{ax})}{p_a}\right) \\ 
&\text{subject to} && \max_{c \in [k]} \sum_{a=1}^k p_a \exp\left(\eta\left(\cL_{ax} - \cL_{cx} - \frac{\ip{q - e_c,\,  G(a, \Phi_{ax})}}{p_a} \right)\right) \leq \exp(\lambda \eta^2) \\
&\text{and}        && \eta \norm{G(a, \sigma)}_\infty \leq p_a \text{ for all $a$ and $\sigma$}\,.
\end{alignedat}
\end{equation}
\end{mdframed}

The optimisation problem in \cref{eq:opt-hp} is not convex, but the solution can be approximated efficiently within a factor of two.
Let $\opthp_q(\eta, \lambda)$ be the optimal value of \cref{eq:opt-hp} with a fixed value of $\lambda$, which is convex.
A larger value of $\lambda$ leads to a larger constraint set and hence $\lambda \mapsto \opthp_q(\eta, \lambda) - \lambda$ is decreasing.
Then the bisection method can be used to find (approximately) the value of $\lambda$ such that $\opthp_q(\eta, \lambda) = 2\lambda$ and you can
check that for this choice $\opthp_q(\eta, \lambda) \leq 2 \opthp_q(\eta)$.
We also define
\begin{align*}
\opthp_*(\eta) &= \sup_{q \in \cP} \opthp_q(\eta)\,.
\end{align*}
The algorithm is exactly the same as \cref{alg:simple} except that the optimisation problem in \cref{eq:opt-hp} is used instead of \cref{eq:opt}.

\begin{algorithm}[h!]
\begin{simplealg}

\textbf{input:} $\eta$

\textbf{for} $t = 1,\ldots,n$: 

\algind Compute $\displaystyle Q_{ta} = \frac{\sind_\Pi(a) \exp\left(-\eta \sum_{s=1}^{t-1} \hat y_{sa}\right)}{\sum_{b \in \Pi} \exp\left(-\eta \sum_{s=1}^{t-1} \hat y_{sb}\right)}$ \\[0.1cm]

\algind Solve (\ref{eq:opt-hp}) with $q = Q_t$ to find $\lambda_t \in \R$ and $P_t \in \cP$ and $G_t \in \cH$ \\[0.1cm]

\algind Sample $A_t \sim P_t$ and observe $\sigma_t$ \\[0.1cm]

\algind Compute $\displaystyle \hat y_t = \frac{G_t(A_t, \sigma_t)}{P_{tA_t}}$ 
\end{simplealg}
\caption{Exponential weights for partial monitoring}
\label{alg:hp}
\end{algorithm}

\begin{theorem}\label{thm:hp}
With probability at least $1 - 2\delta$ the regret of \cref{alg:hp} is bounded by
\begin{align*}
\Reg_n \leq \frac{\log(k) + 2\log(1/\delta)}{\eta} + \eta n \opthp_*(\eta) \,.
\end{align*}
\end{theorem}

\begin{proof}
Let $(\lambda_t)_{t=1}^n$ be the sequence of real values as defined in the algorithm.
Using \cref{lem:hp}, the regret is bounded with probability at least $1 - \delta$ by
\begin{align*}
\Reg_n 
&= \sum_{t=1}^n \left(\cL_{A_t x_t} - \cL_{a^* x_t}\right) 
\leq \frac{\log(1/\delta)}{\eta} + \eta \sum_{t=1}^n \lambda_t + \sum_{t=1}^n \sum_{b=1}^k Q_{tb} \left(\hat y_{tb} - \hat y_{ta^*}\right)\,.
\end{align*}
The second sum is bounded as in the proof of \cref{thm:main} using \cref{thm:exp} by
\begin{align*}
\sum_{t=1}^n \sum_{b=1}^k Q_{tb}\left(\hat y_{tb} - \hat y_{ta^*}\right) 
&\leq \frac{\log(k)}{\eta} + \frac{1}{\eta} \sum_{t=1}^n \Psi_q\left(\eta \hat y_t\right)\,.
\end{align*}
Let $\E_{t-1}[\cdot]$ denote the expectation conditioned on the history observed after $t-1$ rounds.
The last constraint in \cref{eq:opt-hp} ensure that $|\eta \hat y_t| \leq 1$.
Therefore $\Psi_q(\eta \hat y_t) \in [0,1]$ and
\begin{align*}
\E_{t-1}\left[\exp\big(\Psi_{Q_t}(\eta \hat y_t) - \E_{t-1}[\Psi_{Q_t}(\eta \hat y_t)]\big)\right]
&\leq 1 + \E_{t-1}\left[(\Psi_{Q_t}(\eta \hat y_t))^2\right]
\leq \exp\left(\E_{t-1}\left[\Psi_{Q_t}(\eta \hat y_t)\right]\right)\,,
\end{align*}
where we used that $\exp(x) \leq 1 + x + x^2$ for $x \leq 1$ and that $\E[X^2] \leq \E[X]$ for random variables $X \in [0,1]$ and finally that
$1 + x \leq \exp(x)$.
Hence, another application of \cref{lem:hp} shows that with probability at least $1 - \delta$,
\begin{align*}
\frac{1}{\eta}\sum_{t=1}^n \Psi_{Q_t}\left(\eta \hat y_t\right) 
&\leq \frac{2}{\eta} \sum_{t=1}^n \E_{t-1}\left[\Psi_{Q_t}\left(\eta \hat y_t\right)\right] + \frac{1}{\eta}\log\left(\frac{1}{\delta}\right) \\
&\leq \eta \sum_{t=1}^n \left(\opthp_{Q_t}(\eta) - \lambda_t\right) + \frac{1}{\eta}\log\left(\frac{1}{\delta}\right)\,.
\end{align*}
Combining the pieces shows that the regret is bounded with probability at least $1 - 2\delta$ by
\begin{align*}
\Reg_n 
&\leq \frac{\log(k) + 2\log(1/\delta)}{\eta} + \eta \sum_{t=1}^n \opthp_{Q_t}(\eta) 
\leq \frac{\log(k) + 2\log(1/\delta)}{\eta} + \eta n \opthp_*(\eta)\,. \qedhere
\end{align*}
\end{proof}

\begin{remark}
\cref{alg:hp} can be modified with a little effort to adapt the learning rate in a similar manner as \cref{alg:adaptive}. 
The analysis remains more-or-less the same except a version of \cref{lem:hp} must be proven for decreasing sequences of learning rates.
\end{remark}

\paragraph{Applications}
Like $\opt_q(\eta)$, the quantity $\opthp_q(\eta)$ is game-dependent. 
In all the applications that we know of the stability component of the optimisation problem in \cref{eq:opt}
can be bounding by choosing $p$ and $G$ so that the loss estimators do not have magnitude larger than $1/\eta$ and then
using the bounds on $\Psi_q$ in \cref{eq:psi}. The following lemma extracts the core assumptions needed for this argument. Afterwards we
give applications for full information, bandit and partial monitoring games.

\begin{lemma}\label{lem:hp-helper}
Let $q \in \cP$ and $\eta \in (0,1/2)$ and
suppose there exists a $p \in \cP$ and $G \in \cHu$ and $\varphi \in \R^k$ with $\varphi \geq \zeros$ such that for all actions $a$ and $b$ and outcomes $x$,
\begin{align}
\frac{\eta}{p_a} \norm{G(a, \Phi_{ax})}_\infty \leq \frac{1}{2} \qquad \text{and} \qquad
\sum_{a=1}^k \frac{G(a, \Phi_{ax})^2}{p_a} \leq \varphi \qquad \text{and}\qquad \eta^2 \varphi \leq \frac{\ones}{2}\,.
\label{eq:hp-cond}
\end{align}
Then $\opthp_q(\eta) \leq 1 + 12\ip{q, \varphi} + \frac{1}{\eta} \max_{x \in [d]} (p - q)^\top \cL e_x$. 
\end{lemma}

\begin{proof}
Let $G'(a, \sigma) = G(a, \sigma) - 3\eta p_a \varphi$. At heart this is the same biased loss estimator used by \cite{ACFS95} and
generalised by \cite{AR09}. 
Define
\begin{align*}
\lambda = 1 + 6\ip{q, \varphi} + \frac{1}{\eta} \max_{x \in [d]} (p - q) \cL e_x\,.
\end{align*}
We now show that $G'$, $p$ and $\lambda$ satisfies the constraints in \cref{eq:opt-hp} and then that they provide a witness
to the claimed upper bound on $\opthp_q$. 
For any action $c$ and outcome $x$,
\begin{align}
&\sum_{a=1}^k p_a \exp\left(\eta \left(\cL_{ax} - \cL_{cx} - \frac{\ip{q - e_c, G'(a, \Phi_{ax})}}{p_a}\right)\right) \nonumber \\
&\qquad= \exp\left(3 \eta^2 \ip{q - e_c, \varphi} - \eta \cL_{cx}\right) \sum_{a=1}^k p_a \exp\left(\eta \left(\cL_{ax} - \frac{\ip{q - e_c, G(a, \Phi_{ax})}}{p_a}\right)\right)\,.
\label{eq:hp1}
\end{align}
The second term is bounded using $\exp(x) \leq 1 + x + x^2$ for $x \leq 1$,
\begin{align*}
&\sum_{a=1}^k p_a \exp\left(\eta \left(\cL_{ax} - \frac{\ip{q - e_c, G(a, \Phi_{ax})}}{p_a}\right)\right) \\
&\qquad\leq 1 + \eta\cL_{cx} + \eta (p - q)^\top \cL e_x + \eta^2 \sum_{a=1}^k p_a \left(\cL_{ax} - \frac{\ip{q - e_c, G(a, \Phi_{ax})}}{p_a}\right)^2 \\
&\qquad\leq 1 + \eta\cL_{cx} + \eta (p - q)^\top \cL e_x + 3\eta^2\left(1 + \sum_{a=1}^k \frac{\ip{q, G(a, \Phi_{ax})^2}}{p_a} + \sum_{a=1}^k \frac{\ip{e_c, G(a, \Phi_{ax})^2}}{p_a}\right) \\
&\qquad\leq\exp\left(\eta \cL_{cx} + \eta (p - q)^\top \cL e_x + 3\eta^2(1 + \ip{q + e_c, \varphi})\right)\,,
\end{align*}
where in the first inequality we used the fact that $G \in \cH_{\circ}$ is unbiased.
In the second we used that $(x + y + z)^2 \leq 3x^2 + 3y^2 + 3z^2$.
In the last we used the assumptions on $\varphi$ and $1 + x \leq \exp(x)$.
Combining with \cref{eq:hp1} shows that
\begin{align*}
\sum_{a=1}^k p_a \exp\left(\eta \left(\cL_{ax} - \cL_{cx} - \frac{\ip{q - e_c, G'(a, \Phi_{ax})}}{p_a}\right)\right) 
\leq \exp(\lambda \eta^2)\,,
\end{align*}
which confirms that the first constraint in \cref{eq:opt-hp} is satisfies for this choice of $G$ and $p$.
The second constraint in \cref{eq:opt-hp} is satisfied from the assumptions of the lemma:
\begin{align*}
\frac{\eta}{p_a} \norm{G'(a, \sigma)}_\infty
\leq \frac{\eta}{p_a} \norm{G(a, \sigma)}_\infty + \eta^2 \norm{\varphi}_\infty \leq 1\,.
\end{align*}
For the objective we have
\begin{align*}
\lambda + \frac{2}{\eta^2}\sum_{a=1}^k p_a \Psi_q\left(\frac{\eta G'(a, \Phi_{ax})}{p_a}\right)
&\leq \lambda + 2\sum_{a=1}^k \frac{\ip{q, G'(a, \Phi_{ax})^2}}{p_a} \\
&\leq \lambda + 4 \sum_{a=1}^k \frac{\ip{q, G(a, \Phi_{ax})^2}}{p_a} + 4 \eta^2 \ip{q, \varphi^2} \\
&\leq \lambda + 6 \ip{q, \varphi} \\
&= 1 + 12 \ip{q, \varphi} + \frac{1}{\eta} \max_{x \in [d]} (p - q)^\top \cL e_x\,. \qedhere
\end{align*}
\end{proof}

Using the same analysis as in the proofs of \cref{prop:full,prop:bandit,prop:easy,prop:hard} you can prove all the bounds in \cref{tab:hp}.
For example, \cref{lem:hp-helper} can be applied to  full information games by defining $G(a, \sigma) = p_a \sigma$ and $p = q$ and $\varphi = \ones$.
Then $\ip{q, \varphi} = 1$ and for $\eta \leq 1/2$ it follows that $\opthp_*(\eta) \leq 13$. And hence the familiar bound of $O(\sqrt{n \log(k/\delta)})$ 
is recovered using \cref{thm:hp}.
For bandits choose $G(a, \sigma) = e_a \sigma$ and $p = (1 - \gamma )q + \gamma \ones / k$ with $\gamma = k \eta$ and $\varphi = 1/p$.
Then $\opthp_*(\eta) \leq 1 + 13k$ and the regret is $O(\sqrt{nk \log(k/\delta)})$ as expected.

\begin{remark}
The bounds in \cref{tab:hp} are obtained by tuning the learning rate in a manner that depends on $\delta$. The learning rate can be tuned without the
knowledge of $\delta$, but then the dependence on $\log(1/\delta)$ moves outside the square root, a price that is known to be unavoidable \citep{GL16}.
\end{remark}

\begin{table}[h!]
\centering
\renewcommand{\arraystretch}{1.9}
\small
\begin{tabular}{|p{4.3cm}|p{6cm}|}
\hline 
\rowcolor{black!10!white}
 \textbf{Game type} & \textbf{Regret} \\ \hline
\cellcolor{black!10!white} \textsc{Full information} & $O\left(\sqrt{n \log(k/\delta)}\right)$ \\
\cellcolor{black!10!white} \textsc{Bandit} & $O\left(\sqrt{nk \log(k/\delta)}\right)$ \\
\cellcolor{black!10!white} \textsc{Globally observable} & $O\left((c_{\cG} n)^{2/3} \log(k/\delta)^{1/3}\right)$ \\
\cellcolor{black!10!white} \textsc{Locally observable \newline non-degenerate} & $O\left(mk^{3/2} \sqrt{n \log(k/\delta)}\right)$ \\ \hline
\end{tabular}
\caption{High probability regret upper bounds that hold for a given $\delta \in (0,1)$. 
The constant $c_{\cG}$ is game-dependent and can be exponentially large in $d$, which we believe is unavoidable.
}\label{tab:hp}
\end{table}

\section{Water transfer operator}\label{app:water}
Here we provide a simple proof of \cref{lem:water}.
Let $\cT$ be an in-tree over $[k]$.
A vector $y \in \R^k$ is called $\cT$-increasing if $y_a \leq y_b$ for all $(a, b) \in \cT$, which means the function $a \mapsto y_a$ is increasing
towards the root of $\cT$. Similarly, $y$ is $\cT$-decreasing if $y_a \geq y_b$ for all $(a, b) \in \cT$.

\begin{lemma}\label{lem:tree}
Given a tree $\cT$ over $[k]$ and $q \in \cP$, there exists an $r \in \cP$ such that:
\begin{enumerate}
\item [(a)] $r \geq q/k$.
\item [(b)] $r$ is $\cT$-increasing.
\item [(c)] $\ip{r - q, y} \leq 0$ for all $\cT$-decreasing $y \in \R^k$.
\end{enumerate}
\end{lemma}

\newcommand{\desc}{\operatorname{desc}}

\begin{proof}
Let $\desc_\cT(a)$ be the descendants of $a$ in $\cT$ with the convention that $a \in \desc_\cT(a)$. Define $d_\cT(a)$ as the depth of $a$ in $\cT$ with $d_\cT(\troot_\cT) = 1$.
Define $r_a = \sum_{b \in \desc_\cT(a)} q_b / d_\cT(b)$, which is illustrated in \cref{fig:water}.
That $r \in \cP$ follows since
\begin{align*}
\sum_{a=1}^k \sum_{b \in \desc_\cT(a)} \frac{q_b}{d_\cT(b)}
&= \sum_{b=1}^k \frac{q_b}{d_\cT(b)} \sum_{a=1}^k \one{b \in \desc_\cT(a)}
= \sum_{b=1}^k q_b 
= 1\,.
\end{align*}
Part (a) follows because $d_\cT(b) \leq k$.
Part (b) follows because if $(a, b) \in \cT$, then $\desc_\cT(a) \subset \desc_\cT(b)$.
For the last part, the fact that $y$ is $\cT$-decreasing means that
\begin{align*}
\ip{r, y} 
&= \sum_{a=1}^k y_a \sum_{b \in \desc_\cT(a)} \frac{q_b}{d_\cT(b)} 
\leq \sum_{a=1}^k \sum_{b \in \desc_\cT(a)} \frac{y_b q_b}{d_\cT(b)} 
= \sum_{b=1}^k \frac{y_b q_b}{d_{\cT(b)}} \sum_{a=1}^k \one{b \in \desc_\cT(a)} 
= \ip{q, y}\,.
\end{align*}
Rearranging completes the proof.
\end{proof}

\begin{figure}[h!]
\centering
\includegraphics[width=14cm]{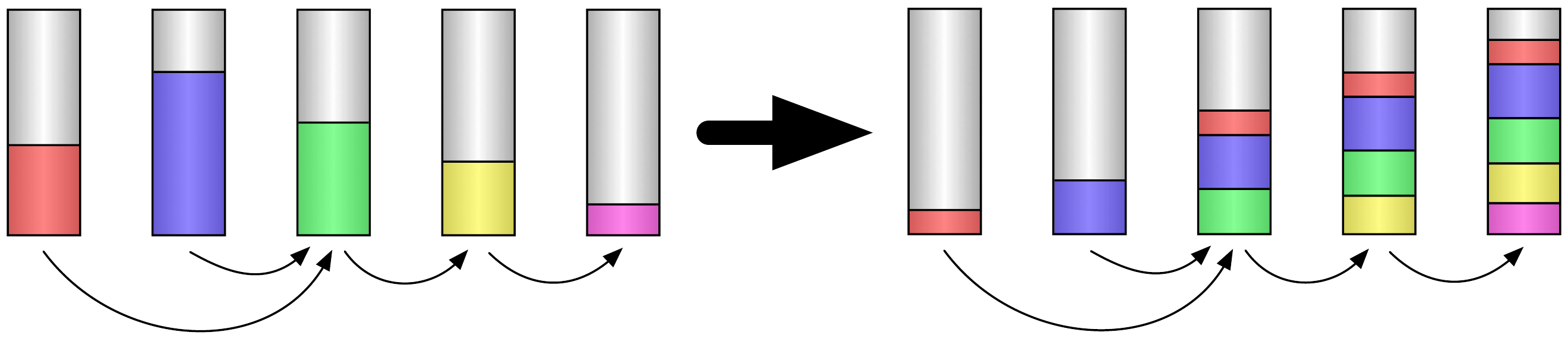}
\caption{Illustration of $r$ as defined in the proof of \cref{lem:water}.}\label{fig:water}
\end{figure}

\newcommand{\parent}{\operatorname{par}}
\newcommand{\ri}{\operatorname{ri}}

\begin{proof}[Proof of \cref{lem:water}]
The result follows from \cref{lem:tree} and by proving there exists an in-tree $\cT$ over $[k]$ such that $\cL \nu$ is $\cT$-decreasing.
We start by proving the existence of $\cT$ when $\nu \in \ri(C_{a^*_\nu})$ for some Parent optimal action $a^*_\nu$.
Define a function $\parent : [k] \to [k]$ by
\begin{align*}
\parent(a) = \argmin_{b : (a, b) \in \cE} e_b^\top \cL \nu\,,
\end{align*}
where the ties in the $\argmin$ are broken arbitrarily.
We will shortly show that $(e_a - e_{\parent(a)})^\top \cL \nu > 0$ for all $a \neq a^*_\nu$, which means that 
$\cT = \{(a, \parent(a)) : a \neq a^*_\nu\}$
is an in-tree over $[k]$ on which $\cL^\top\nu$ is $\cT$-decreasing.
Let $D = \{(a, b, c) : \dim(C_a \cap C_b \cap C_c) \leq d - 3\}$
and 
\begin{align*}
A = \bigcup_{a,b,c \in D} C_a \cap C_b \cap C_c\,.
\end{align*}
Let $a \neq a^*_\nu$ and $\mu \in \ri(C_a)$ be such that the chord connecting $\mu$ and $\nu$ does not intersect $A$.
Next, let $\rho \in \partial C_a$ be such that $\rho - \mu$ is proportional to $\nu - \mu$ and $b \neq a$ be an action with $\rho \in C_b$.
Since $\mu \in \ri(C_a)$ we have $e_a^\top \cL \mu < e_b^\top \cL \mu$ and since $\rho \in C_a \cap C_b$ we have $e_a^\top \cL \rho = e_b^\top \cL \rho$. 
Hence $e_b^\top \cL \nu < e_a^\top \cL \nu$.
The choice of $\mu$ ensures that $\rho \notin A$ and hence $(a, b) \in \cE$, which means that $\parent(a)$ is well defined and satisfies the claimed monotonicity conditions.
Suppose now that $\nu$ is arbitrary and $a^*_\nu \in C_\nu$.
Then take a sequence $(\nu_t)_{t=1}^\infty$ converging to $\nu$ and with $\nu_t \in \ri(a^*_\nu)$. 
By the previous argument there exists a sequence of in-trees $(\cT_t)_{t=1}^\infty$ such that $\cL \nu_t$ is $\cT_t$-decreasing.
Since the space of trees is finite, the sequence $(\cT_t)_{t=1}^\infty$ has a cluster point $\cT$ and it is easy to see that $\cL \nu$ is $\cT$-decreasing.
\end{proof}

\section{Bounds on the estimation functions}

The polynomial dependence on $k$ and $m$ in locally observable non-degenerate games follows from the simple combinatorial structure when loss differences
are estimated by playing two actions only. We provide the following lemma, which strengthens slightly our previous result \citep{LS18pm}.

\begin{lemma}\label{lem:bounds}
If $\cG = (\Phi, \cL)$ is locally observable and non-degenerate and actions $(a, b) \in \cE$ are neighbours, then there exist functions $w_a, w_b : \Sigma \to \R$ such that
$\norm{w_a}_\infty \leq m/2$ and $\norm{w_b}_\infty \leq m/2$ and
\begin{align}
\cL_{ax} - \cL_{bx} = w_a(\Phi_{ax}) + w_b(\Phi_{bx}) \text{ for all } x \in[d]\,. 
\label{eq:w-cond}
\end{align}
\end{lemma}

\begin{proof}[Proof of \cref{lem:bounds}]
By the definition of local observability and non-degeneracy there exists $w_a, w_b$ satisfying \cref{eq:w-cond}.
Consider the bipartite graph over $V = \{(a, 1),\ldots,(a, m), (b,1),\ldots,(b,m)\}$ and edges between vertices $(a, \sigma)$ and $(b, \sigma')$ if there exists an $x \in [d]$ such
that $\Phi_{ax} = \sigma$ and $\Phi_{bx} = \sigma'$.
Define a function $f : V \to \R$ by $f((a, \sigma)) = w_a(\sigma)$ and $f((b, \sigma)) = w_b(\sigma)$.
Since entries in the loss matrix are bounded in $[0,1]$ it holds that $f(w) + f(v) \in [0,1]$ for all edges $(w, v)$. 
The result follows from \cref{lem:graph}.
\end{proof}

For degenerate games the learner may need more than two actions to produce unbiased loss estimates, which unfortunately introduces the potential for an unpleasant combinatorial
structure that makes learning much harder. Nevertheless, the norm of the estimation vectors can be uniformly bounded in terms of $d$ and $k$.

\begin{proposition}\label{lem:bounds-global}
Suppose that $(\cL, \Phi)$ is globally observable and $a$ and $b$ are neighbours. Then there exists a function $\est : [k] \times \Sigma \to \R$ such that for all $x \in [d]$,
\begin{align*}
\sum_{c=1}^k \est(c, \Phi_{cx}) = \cL_{ax} - \cL_{bx}\,.
\end{align*}
Furthermore, $\est$ can be chosen so that $\norm{\est}_\infty \leq d^{1/2} k^{d/2}$.
\end{proposition}

\begin{proof}
For action $a$, let $S_a \in \{0,1\}^{|\Sigma| \times d}$ be the matrix with $(S_a)_{\sigma x} = \one{\Phi_{ax} = \sigma}$, which 
means that $S_a e_x = e_{\Phi_{ax}}$. Here we have abused notation by indexing the rows of $S_a$ using signals.
Let $S = (S_1^\top,\ldots,S_k^\top)$, which means that $S \in \R^{d \times mk}$.
Then let $y = (e_a - e_b)^\top \cL \in [-1,1]^k$. We identify $w$ with a vector in $\R^{km}$. 
By the assumption of global observability there exists a $w \in \R^{km}$ such that $S w = y$.
Hence we may take $w = S^+ y$ with $S^+$ the Moore-Penrose pseudo-inverse and for which
\begin{align*}
\norm{w}_\infty \leq \norm{w}_2 \leq \Vert S^+ \Vert_2 \norm{y}_2 \leq d^{1/2} \Vert S^+ \Vert_2 \leq d^{1/2} k^{d/2}\,,
\end{align*}
where $\norm{S^+}_2$ is the spectral norm of $S^+$ and the final inequality follows from \cref{lem:spec}.
\end{proof}

\section{Non-equivalence to Hedge}\label{app:hedge}

\cref{alg:simple} does not reduce to Hedge in the full information setting.
The full information game with binary losses and 
$k$ actions has $d = 2^k$ outcomes, which we associate with $\{0,1\}^k$ via some arbitrary bijection and then view the outcomes as being in $\{0,1\}^k$ instead of $[d]$.
The signal matrix is $\Phi_{ax} = x \in \{0,1\}^k$ and the loss matrix is $\cL_{ax} = x_a$.
Given distribution $q \in \cP$, the estimation function $G \in \cH$ that minimises the objective in \cref{eq:opt} for the full information game can be calculated analytically as
\begin{align*}
G(a, \sigma) = p_a (\sigma + c(\sigma))\,,
\end{align*}
where the shifting constant $c(\sigma)$ is given by
\begin{align*}
c(\sigma) = \frac{1}{\eta} \log\left(\ip{q, \exp(-\eta \sigma)}\right) = -\ip{q, \sigma} + O(\eta)\,.
\end{align*}
Note that $G \in \cHu$ is unbiased.
The sampling distribution $p$ should be the minimiser of
\begin{align*}
&\frac{1}{\eta} \min_{p \in \cP} \max_{x \in \{0,1\}^k} \left(\ip{p - q, x} + \frac{1}{\eta} \left(\eta \ip{q, x} + \log\left(\ip{q, \exp(-\eta x)}\right)\right)\right) \\
&\qquad\qquad\qquad\approx \frac{1}{\eta} \min_{p \in \cP} \max_{x \in \{0,1\}^k} \left(\ip{p - q, x} + \frac{\eta}{2} \ip{q, x^2}\right)\,.
\end{align*}
The inner optimisation problem is not especially pleasant, but as $\eta$ tends to zero the linear term dominates and the optimal $p$ tends to $q$. Generally speaking, however, the optimal $p$ is
not equal to $q$. A numerical calculation shows that when $k = 2$ and $\eta = 0.5$ and $q = (0.9, 0.1)$, then the optimal $p$ is approximately $p = (0.897, 0.103) \neq q$.

\section{Technical lemmas}

\begin{lemma}[\citealt{PL19}]\label{lem:tech1}
Let $(a_t)_{t=1}^n$ be a sequence of non-negative reals. Then
\begin{align*}
\sum_{t=1}^n \frac{a_t}{\sqrt{1 + \sum_{s=1}^{t-1} a_s}} \leq 4 \sqrt{1 + \frac{1}{2} \sum_{t=1}^n a_t} + \max_{t \in [n]} a_t \,.
\end{align*}
\end{lemma}

\begin{lemma}\label{lem:tech2}
Let $\alpha > 0$ and $(a_t)_{t=1}^n$ be a sequence of non-negative reals with $a_{t+1} \leq a_t + \alpha a_t^{1/4}$.
Then 
\begin{align*}
a_n \leq \left(\frac{3 \alpha (n-1)}{4} + a_1^{3/4}\right)^{4/3}\,.
\end{align*}
\end{lemma}

\begin{proof}
Consider the differential equation $y(0) = a_1$ and
$y'(t) = \alpha y(t)^{1/4}$, which has solution
\begin{align*}
y(t) =  \left(\frac{3\alpha t}{4} + a_1^{3/4}\right)^{4/3}\,.
\end{align*}
By comparison, $a_n \leq y(n-1)$ and the result follows.
\end{proof}

The next lemma provides a lower bound on the smallest non-zero eigenvalue of a positive semi-definite matrix with integer entries.
Such results are somehow the reverse of the more well-known Hadamard problem of finding the maximum determinant \citep{AV97}. 
Presumably the naive bound below is known to experts, but a source seems hard to find.

\begin{lemma}\label{lem:spec}
Let $k \geq 3$ and $A \in \{0,\ldots,k\}^{d \times d}$ be non-zero and positive semi-definite with eigenvalues $\lambda_1,\ldots,\lambda_d$. Then
$\min\{\lambda_i : \lambda_i > 0\} \geq k^{-d}$.
\end{lemma}

\begin{proof}
Assume without loss of generality that $(\lambda_j)_{j=1}^d$ is decreasing and $i$ is the index of the smallest non-zero eigenvalue.
If $i = 1$, then $\lambda_i \geq 1$ and the result is immediate. Suppose now that $i > 1$. 
Since $A$ has integer coefficients, the product of its non-zero eigenvalues is a positive integer, which means that
$\prod_{i: \lambda_i > 0} \lambda_i \geq 1$. 
Hence, by the arithmetic-geometric mean inequality,
\begin{align*}
\frac{1}{\lambda_i} 
&\leq \prod_{j=1}^{i-1} \lambda_j 
\leq \left(\frac{1}{i-1} \sum_{j=1}^{i-1} \lambda_j\right)^{i-1} 
\leq \left(\frac{\tr(A)}{i-1}\right)^{i-1} 
\leq \left(\frac{dk}{i-1}\right)^{i-1}
\leq k^d\,.
\qedhere
\end{align*}
\end{proof}

\begin{lemma}\label{lem:graph}
Let $V_1$ and $V_2$ be disjoint sets with $|V_1| = |V_2| = m$ and $V = V_1 \cup V_2$.
Suppose that $(V, E)$ is a bipartite graph with $E \subseteq V_1 \times V_2$ and $f : V \to \R$ is a function such that $f(u) + f(v) \in [0,1]$ for all $(u, v) \in E$. 
Then there exists a function $g : V \to \R$ such that
\begin{enumerate}
\item[(a)] $\norm{g}_\infty \leq \frac{m}{2}$.
\item[(b)] $g(u) + g(v) = f(u) + f(v)$ for all $u, v \in E$.
\end{enumerate}
\end{lemma}

\begin{proof}
We define $g$ on each connected component of $(V, E)$.
For edge $(u, v) \in E$ we abuse notation by writing $f(e) = f(u) + f(v)$.
Let $U \subseteq V$ be a connected component and $u = \argmin_{v \in U \cap V_1} f(v)$.
\begin{align*}
g(v) = 
\begin{cases}
f(v) - f(u) - m/2 + 1, & \text{if } v \in V_1; \\
f(v) + f(u) + m/2 - 1, & \text{if } v \in V_2\,.
\end{cases}
\end{align*}
Then for any $v \in U \cap V_1$ there exists a path $(e_t)_{t=1}^n$ from $v$ to $u$ with $n \leq 2(m-1)$ and
\begin{align*}
g(v) + m/2 - 1 = g(v) - g(u) = \sum_{t=1}^n (-1)^{t+1} f(e_t) \leq m - 1\,.
\end{align*}
Hence $g(v) \in [-m/2+1, m/2]$ for all $v \in U \cap V_1$ and so $g(v) \in [-m/2, m/2]$ for all $v \in U \cap V_2$.
\end{proof}

The following lemma has been seen before in many forms \citep[for example]{ACFS95} and follows immediately from the Chernoff method.

\begin{lemma}\label{lem:hp}
Suppose that $(X_t)_{t=1}^n$ is a sequence of random variables adapted to filtration $(\cF_t)_{t=1}^n$ and $(\lambda_t)_{t=1}^n$ is $(\cF_t)$-predictable and for $\eta > 0$,
\begin{align*}
\E[\exp(\eta X_t - \lambda_t^2) \mid \cF_{t-1}] \leq 1 \,\,\, a.s.\,.
\end{align*}
Then for any $\delta \in (0,1)$,
\begin{align*}
\Prob{\sum_{t=1}^n X_t \geq \sum_{t=1}^n\frac{\lambda_t^2}{\eta} + \frac{\log(1/\delta)}{\eta}} \leq \delta\,.
\end{align*}
\end{lemma}

\begin{proof}
By Markov's inequality and the tower rule for conditional expectation,
\begin{align*}
\Prob{\exp\left(\sum_{t=1}^n \eta X_t - \lambda_t^2\right) \geq \frac{1}{\delta}} \leq \delta\,.
\end{align*}
Re-arranging completes the proof. 
\end{proof}

\section{A second-order cone approximation}\label{sec:approx}

The optimisation problem in \cref{eq:opt} is convex and can be written as an exponential cone program.
For small problems and reasonably large $\eta$ it is amenable to standard methods. Numerical instability seems to be a problem when $\eta$ is small, however.
A practical resolution is to move some of the analysis into the optimisation problem by adding constraints on the magnitude of the estimation function and
then approximating $\Psi_q$ by an upper bound as in \cref{eq:psi}.
This leads to the following formulation of the approximation of \cref{eq:opt} as a second-order cone program:
\begin{mdframed}[roundcorner=1pt,backgroundcolor=black!5!white]
\begin{equation}
\label{eq:socp}
\begin{alignedat}{3}
&\underset{G \in \cH, p \in \cP}{\text{minimise}} \qquad && \max_{x \in [d]} \Bigg[
  \frac{(p - q)^\top \cL e_x + \bias_q(G ; x)}{\eta} + \sum_{a=1}^k \frac{\ip{q, G(a, \Phi_{ax})^2}}{p_a} 
  \Bigg] \\
&\text{subject to} && G(a, \sigma) + \frac{p_a}{\eta} \ones \geq \zeros \text{ for all } a \text{ and } \sigma \\
&\text{and}        && p_a \geq \epsilon \text{ for all } a\,.
\end{alignedat}
\end{equation}
\end{mdframed}
The first constraint justifies using the bound in \cref{eq:psi} to approximate $\Psi_q$.
The parameter $\epsilon \geq 0$ in the second constraint is present to improve numerical stability and should be chosen so that
its impact on the regret is negligible. For example, $\epsilon = \eta^2$.  

Let $\opt_q^\sim(\eta)$ be the
optimal value of the above optimisation problem and
\begin{align*}
\opt_*^\sim(\eta) = \sup_{q \in \cP} \opt_q^\sim(\eta)\,.
\end{align*}
It is straightforward to show that the value of \cref{eq:opt} at the optimiser of \cref{eq:socp} is at most $\opt_q^\sim(\eta)$.
Indeed, the upper bounds on $\opt_q(\eta)$ were all proven in this manner. We are not aware of a situation where $\opt_q(\eta) \ll \opt_q^\sim(\eta)$.

\begin{figure}[h!]
\pgfplotstableread{data/pennies-c.txt}{\DataPenniesC}
\centering
\begin{tikzpicture}
\begin{axis}[width=10cm,xlabel={$c$},legend pos=north west,legend cell align={left}]
\addplot+[ultra thick,mark=none] table[x index=0,y index=1] {\DataPenniesC};
\addlegendentry{$\opt_*(\eta)$};
\addplot+[ultra thick,mark=none] table[x index=0,y index=3] {\DataPenniesC};
\addlegendentry{$\opt_*^\sim(\eta)$};
\addplot+[ultra thick,mark=none,black] table[x index=0,y index=4] {\DataPenniesC};
\addlegendentry{\cref{eq:opt} at optimiser of \cref{eq:socp}};
\end{axis}
\end{tikzpicture}
\caption{
The plot illustrates the quality of the approximation in \cref{eq:socp} for the matching pennies game with the cost varying on the $x$-axis and a fixed learning rate: $\eta = 0.01$.
The blue line shows $\opt_*(\eta)$.
The red line shows $\opt_*^\sim(\eta)$ and the black line is the value of \cref{eq:opt} evaluated at the optimiser of \cref{eq:socp}.
At least for this game the approximation is quite reasonable. The abrupt increase when $c > 1/2$ occurs because this is where the game
transitions from being locally observable to only globally observable.
Both \cref{eq:opt} and \cref{eq:socp} were solved using the Splitting Cone Solver \citep{DCP16,DCP17}.
}
\end{figure}
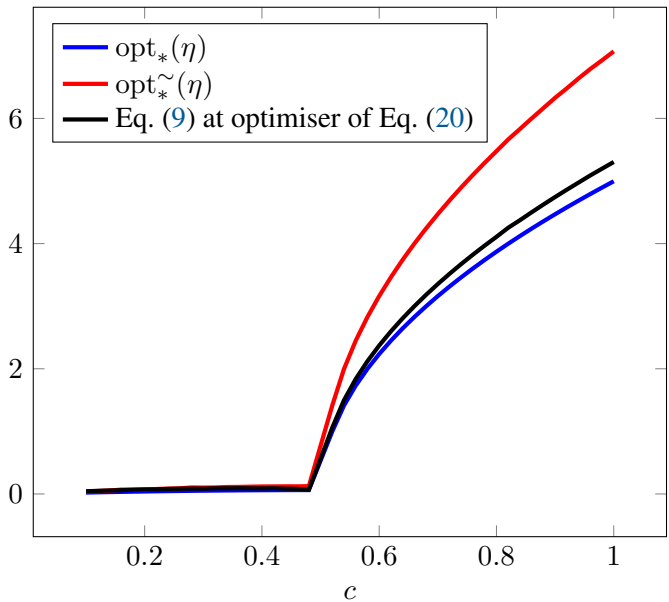

\section{Experiments}

In our simple experiments we use the Splitting Cone Solver \citep{DCP16,DCP17} to solve the optimisation problem in \cref{eq:socp}.
The performance of the algorithm is illustrated on the costly matching pennies game (\cref{eq:example}), which
is locally observable and non-degenerate for $c \in (0, 1/2)$ and globally observable for $c > 1/2$.
When $c = 1/2$ it is degenerate and locally observable. When $c = 0$ it is trivial.
The next figure shows the regret of ExpPM in costly matching pennies for two different values of $c$.

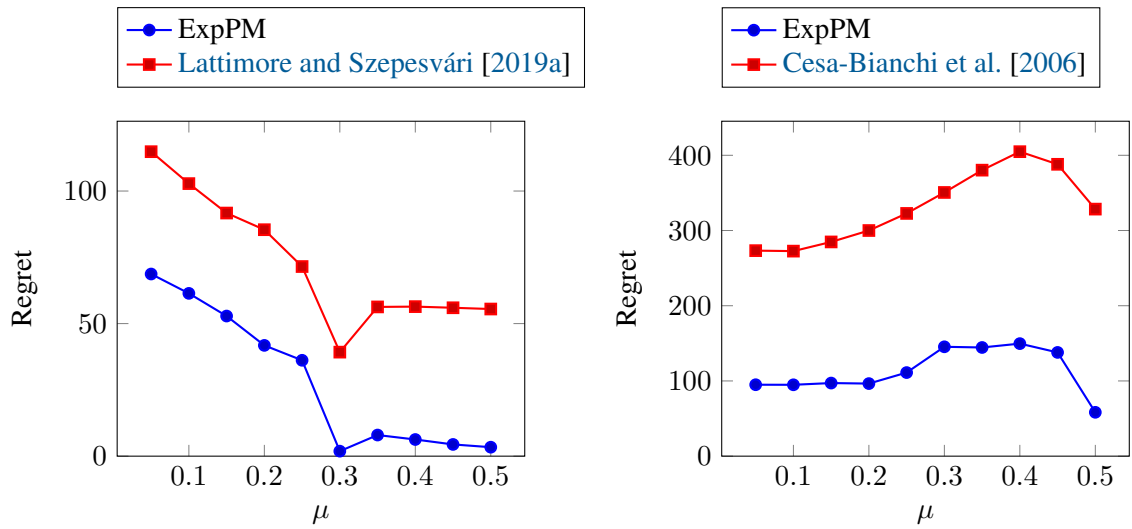
\begin{figure}[h!]
\pgfplotstableread{data/exp_pm-0.3.txt}{\DataExpPMEasy}
\pgfplotstableread{data/nw-0.3.txt}{\DataNWEasy}
\centering
\begin{tikzpicture}
\begin{axis}[width=7cm,xlabel={$\mu$},ylabel={Regret},legend cell align={left},legend style={at={(0,1.1)},anchor=south west},ymin=0]
\addplot+[thick] table[x index=1,y index=3] {\DataExpPMEasy};
\addlegendentry{ExpPM};
\addplot+[thick] table[x index=1,y index=3] {\DataNWEasy};
\addlegendentry{\cite{LS18pm}};
\end{axis}
\end{tikzpicture}
\pgfplotstableread{data/exp_pm-1.txt}{\DataExpPMHard}
\pgfplotstableread{data/gexp-1.txt}{\DataGexpHard}
\centering
\begin{tikzpicture}
\begin{axis}[width=7cm,xlabel={$\mu$},ylabel={Regret},legend cell align={left},legend style={at={(0,1.1)},anchor=south west},ymin=0]
\addplot+[thick] table[x index=1,y index=3] {\DataExpPMHard};
\addlegendentry{ExpPM};
\addplot+[thick] table[x index=1,y index=3] {\DataGexpHard};
\addlegendentry{\cite{CBLuSt06}};
\end{axis}
\end{tikzpicture}
\caption{Costly matching pennies where the adversary is stochastic and samples from the outcomes i.i.d.\ from distribution $(\mu, 1-\mu)$. The horizon is $n = 2000$.
On the left plot the cost is $c = 3/10$ and the algorithm is compared to Neighbourhood Watch 2 \citep{LS18pm}. On the right plot the cost
is $c = 1$ and the algorithm is compared to the algorithm by \cite{CBLuSt06}.}
\end{figure}

\end{document}